\documentclass[conference]{IEEEtran}
\IEEEoverridecommandlockouts
\usepackage{cite}
\usepackage{amsmath,amssymb,amsfonts}
\usepackage{graphicx}
\usepackage{xcolor}

\usepackage[hidelinks]{hyperref}
\usepackage{xspace}
\usepackage[inline]{enumitem}
\usepackage[capitalize]{cleveref}
\usepackage{multirow}
\usepackage{arydshln} 
\usepackage{ifthen}
\usepackage[placement=after]{notes2bib}

\makeatletter
\let\MYcaption\@makecaption
\makeatother

\usepackage[font=footnotesize,skip=3pt,subrefformat=parens]{subcaption}

\makeatletter
\let\@makecaption\MYcaption
\makeatother

\newtheorem{proposition}{Proposition}

\makeatletter
\DeclareRobustCommand\onedot{\futurelet\@let@token\@onedot}
\def\@onedot{\ifx\@let@token.\else.\null\fi\xspace}

\def\ie{\emph{i.e}\onedot} 
 
\def\etc{\emph{etc}\onedot}

\def\etal{et al\onedot}
\makeatother

\def\BibTeX{{\rm B\kern-.05em{\sc i\kern-.025em b}\kern-.08em
    T\kern-.1667em\lower.7ex\hbox{E}\kern-.125emX}}

\DeclareMathOperator{\Real}{Re}

\newcommand{\iter}[2]{#1^{(#2)}}

\newcommand{\twodots}{\mathinner {\ldotp \ldotp}}

\newcommand{\mathC}{\mathbb{C}}
\newcommand{\mathE}{\mathbb{E}}

\newcommand{\mathN}{\mathbb{N}}

\newcommand{\mathR}{\mathbb{R}}

\newcommand{\mathV}{\mathbb{V}}

\newcommand{\mathZ}{\mathbb{Z}}

\newcommand{\MF}{\mathcal F}
\newcommand{\MG}{\mathcal G}
\newcommand{\MH}{\mathcal H}

\newcommand{\ML}{\mathcal L}

\newcommand{\MS}{\mathcal S}

\newcommand{\BA}{\boldsymbol A}

\newcommand{\ba}{\boldsymbol a}
\newcommand{\bb}{\boldsymbol b}

\newcommand{\bk}{\boldsymbol k}

\newcommand{\bn}{\boldsymbol n}

\newcommand{\BBD}{\mathbf D}

\newcommand{\BBV}{\mathbf V}

\newcommand{\BBX}{\mathbf X}
\newcommand{\BBY}{\mathbf Y}

\newcommand{\RMA}{\mathrm{A}}

\newcommand{\RMD}{\mathrm{D}}

\newcommand{\RMV}{\mathrm{V}}
\newcommand{\RMW}{\mathrm{W}}
\newcommand{\RMX}{\mathrm{X}}
\newcommand{\RMY}{\mathrm{Y}}
\newcommand{\RMZ}{\mathrm{Z}}

\newcommand{\balpha}{\boldsymbol \alpha}

\newcommand{\bmu}{\boldsymbol \mu}

\newcommand{\blambda}{\boldsymbol \lambda}

\newcommand{\bvarSigma}{\boldsymbol\varSigma}

\newcommand{\bone}{\boldsymbol 1}

\newcommand{\red}[1]{\textcolor{red}{#1}}

\newcommand{\dtcwpt}{DT-$\mathC$WPT\xspace}
\newcommand{\dtrwpt}{DT-$\mathR$WPT\xspace}

\newcommand{\range}[2]{\left\{#1 \twodots #2\right\}}

\newcommand{\oneto}[1]{\range{1}{#1}}

\newcommand{\interval}[2]{\left[#1,\, #2\right]}

\newcommand{\intervalexclr}[2]{\left[#1,\, #2\right[}

\newcommand{\zeroone}{\interval{0}{1}}

\newcommand{\mpipi}{\interval{-\pi}{\pi}}

\newcommand{\nonzeroset}[1]{#1 \setminus \{0\}}
\newcommand{\nonzeroMathN}{\nonzeroset{\mathN}}

\newcommand{\Expval}{\mathE}
\newcommand{\Var}{\mathV}

\newcommand{\bigO}{O}

\renewcommand{\arraystretch}{1.5} 

\renewcommand{\texttt}[1]{%
	\begingroup
	\ttfamily
	\begingroup\lccode`~=`/\lowercase{\endgroup\def~}{/\discretionary{}{}{}}%
	\begingroup\lccode`~=`[\lowercase{\endgroup\def~}{[\discretionary{}{}{}}%
	\begingroup\lccode`~=`.\lowercase{\endgroup\def~}{.\discretionary{}{}{}}%
	\catcode`/=\active\catcode`[=\active\catcode`.=\active
	\scantokens{#1\noexpand}%
	\endgroup
} 

\newcommand{\norm}[1]{\left\|#1\right\|}

\newcommand{\normtwo}[1]{\norm{#1}_2}
\newcommand{\norminfty}[1]{\norm{#1}_{\infty}}

\newcommand{\bignorm}[1]{\bigl\|#1\bigr\|}
\newcommand{\bignormone}[1]{\bignorm{#1}_1}
\newcommand{\bignormtwo}[1]{\bignorm{#1}_2}
\newcommand{\bignorminfty}[1]{\bignorm{#1}_{\infty}}

\newcommand{\obo}{1 \times 1}

\newcommand{\qqand}{\qquad\mbox{and}\qquad}

\newcommand{\todo}[1][]{%
  	\red{\ifthenelse{\equal{#1}{}}%
    	{[TODO]}%
    	{[TODO: #1]}%
  	}\xspace
}

\newcommand{\rulesep}{\unskip\ \vrule\ }

\makeatletter
\DeclareRobustCommand\onedot{\futurelet\@let@token\@onedot}
\def\@onedot{\ifx\@let@token.\else.\null\fi\xspace}

\def\ie{\emph{i.e}\onedot} 
 
\def\etc{\emph{etc}\onedot}

\def\etal{et al\onedot}
\makeatother

\DeclareMathOperator{\supp}{supp}

\DeclareMathOperator{\relu}{ReLU}

\DeclareMathOperator{\maxpool}{MaxPool}

\DeclareMathOperator{\idxconv}{conv}
\DeclareMathOperator{\idxbn}{bn}
\DeclareMathOperator{\idxbias}{bias}

\DeclareMathOperator{\idxmod}{mod}

\DeclareMathOperator{\idxrelu}{relu}

\DeclareMathOperator{\idxmax}{max}
\DeclareMathOperator{\idxlow}{free}
\DeclareMathOperator{\idxhigh}{gabor}

\DeclareMathOperator{\idxmaxpool}{max}
\DeclareMathOperator{\idxbl}{b}
\DeclareMathOperator{\idxblur}{blur}
\DeclareMathOperator{\idxadablur}{ablur}

\DeclareMathOperator{\idxdtcwpt}{dt}

\DeclareMathOperator{\idxlum}{lum}
\DeclareMathOperator{\idxfilt}{filt}

\DeclareMathOperator{\idxgroup}{g}
\DeclareMathOperator{\idxsum}{s}
\DeclareMathOperator{\idxprod}{p}
\DeclareMathOperator{\idxexp}{e}
\DeclareMathOperator{\idxsoftmax}{sftmx}
\DeclareMathOperator{\idxstd}{std}

\newcommand{\convlayer}{\mbox{Conv}}
\newcommand{\complexConvlayer}{\mathC\convlayer} 
\newcommand{\bnlayer}{\mbox{BN}}
\newcommand{\biaslayer}{\mbox{Bias}}
\newcommand{\subsamplayer}{\mbox{Sub}} 
\newcommand{\blurlayer}{\mbox{Blur}}
\newcommand{\rmaxlayer}{\mathR\mbox{Max}}

\newcommand{\cmodlayer}{\mathC\mbox{Mod}}
\newcommand{\bnzeroerolayer}{\bnlayer 0}
\newcommand{\bnzerolayer}{\bnlayer 0}
\newcommand{\relulayer}{\mbox{ReLU}}
\newcommand{\maxpoollayer}{\mbox{MaxPool}}
\newcommand{\blurpoollayer}{\mbox{BlurPool}}
\newcommand{\adablurpoollayer}{\mbox{ABlurPool}}
\newcommand{\maxlayer}{\mbox{Max}}
\newcommand{\moduluslayer}{\mbox{Modulus}}
\newcommand{\softmaxlayer}{\mbox{Softmax}}

\newcommand{\complexConv}{$\mathC$Conv\xspace}
\newcommand{\rmax}{$\mathR$Max\xspace}
\newcommand{\cmod}{$\mathC$Mod\xspace}
\newcommand{\cmodbn}{BN$0$\xspace}

\newcommand{\wavecnn}{WCNN\xspace}
\newcommand{\cwavecnn}{$\mathC$\wavecnn}
\newcommand{\blurcnn}{BlurCNN\xspace}
\newcommand{\blurwavecnn}{Blur\wavecnn}
\newcommand{\cblurwavecnn}{$\mathC$\blurwavecnn}
\newcommand{\adablurcnn}{A\blurcnn}
\newcommand{\adablurwavecnn}{A\blurwavecnn}
\newcommand{\cadablurwavecnn}{$\mathC$\adablurwavecnn}

\newcommand{\waveblock}{WBlock\xspace}

\newcommand{\cwaveblock}{$\mathC$\waveblock}
\newcommand{\wavecnns}{{\wavecnn}s\xspace}
\newcommand{\cwavecnns}{{\cwavecnn}s\xspace}

\newcommand{\wavealexnet}{WAlexNet\xspace}
\newcommand{\blurwavealexnet}{BlurWAlexNet\xspace}
\newcommand{\cwavealexnet}{$\mathC$\wavealexnet}
\newcommand{\cblurwavealexnet}{$\mathC$\blurwavealexnet}

\newcommand{\waveresnet}{WResNet\xspace}

\newcommand{\cwaveresnet}{$\mathC$\waveresnet}

\newcommand{\xcoord}{$x$\xspace}

\newcommand{\setof}[1]{\oneto{#1}}

\newcommand{\twodseqs}{\MS}
\newcommand{\gaborchannels}{\MG}
\newcommand{\freechannels}{\MF}

\newcommand{\inchannels}{K}
\newcommand{\outchannels}{L}

\newcommand{\sub}{m} 

\newcommand{\imgsize}{N}

\newcommand{\depth}{J}
\newcommand{\ngroups}{Q}
\newcommand{\discretefilterSupportsize}{\kappa}

\newcommand{\outchannelsLow}{\outchannels_{\idxlow}}
\newcommand{\outchannelsHigh}{\outchannels_{\idxhigh}}
\newcommand{\dtchannels}{\inchannels_{\idxdtcwpt}}
\newcommand{\dtchannelsSubset}{\dtchannels'}
\newcommand{\inchannelsGroup}{\inchannels_{\selectGroup}}
\newcommand{\outchannelsGroup}{\outchannels_{\selectGroup}}
\newcommand{\outchannelsPergroup}{\outchannels_{\idxgroup}}

\newcommand{\imgsizebis}{\imgsize'}

\newcommand{\filtersize}{\imgsize_{\idxfilt}}
\newcommand{\blurfiltsize}{\imgsize_{\idxbl}}

\newcommand{\selectInchannel}{k}
\newcommand{\selectOutchannel}{l}
\newcommand{\biasval}{b}
\newcommand{\multval}{a}

\newcommand{\regparam}{\lambda}

\newcommand{\colormixval}{\mu}
\newcommand{\selectGroup}{q}

\newcommand{\featmapmixval}{\alpha}

\newcommand{\biasvalSelect}{\biasval_\selectOutchannel}
\newcommand{\multvalSelect}{\multval_\selectOutchannel}
\newcommand{\colormixvalSelect}{\colormixval_\selectInchannel}
\newcommand{\selectDtchannel}{\selectInchannel}
\newcommand{\featmapmixvalGroup}{\iter{\featmapmixval}{\selectGroup}}
\newcommand{\regparamGroup}{\regparam_\selectGroup}

\newcommand{\computtime}{t}
\newcommand{\Computtime}{T}
\newcommand{\Tensorsize}{S}

\newcommand{\timeSum}{\computtime_{\idxsum}}
\newcommand{\timeProd}{\computtime_{\idxprod}}
\newcommand{\timeExp}{\computtime_{\idxexp}}
\newcommand{\timeMod}{\computtime_{\idxmod}}
\newcommand{\timeReLU}{\computtime_{\idxrelu}}
\newcommand{\timeMax}{\computtime_{\idxmaxpool}}

\newcommand{\TimeConv}{\Computtime_{\idxconv}}
\newcommand{\TimeCConv}{\Computtime_{\mathC\idxconv}}
\newcommand{\TimeBN}{\Computtime_{\idxbn}}
\newcommand{\TimeBias}{\Computtime_{\idxbias}}
\newcommand{\TimeReLU}{\Computtime_{\idxrelu}}
\newcommand{\TimeMax}{\Computtime_{\idxmaxpool}}
\newcommand{\TimeMod}{\Computtime_{\idxmod}}
\newcommand{\TimeBlur}{\Computtime_{\idxblur}}
\newcommand{\TimeAdablur}{\Computtime_{\idxadablur}}

\newcommand{\TensorsizeStd}{\Tensorsize_{\idxstd}}
\newcommand{\TensorsizeBlur}{\Tensorsize_{\idxblur}}
\newcommand{\TensorsizeAdablur}{\Tensorsize_{\idxadablur}}
\newcommand{\TensorsizeCMod}{\Tensorsize_{\idxmod}}

\newcommand{\biasvec}{\bb}
\newcommand{\multvec}{\ba}

\newcommand{\vectorindex}{\bn}
\newcommand{\vectorindexbis}{\bk}
\newcommand{\regparamvec}{\blambda}

\newcommand{\colormixvec}{\bmu}
\newcommand{\permutmatrix}{\bvarSigma}
\newcommand{\featmapmixvec}{\balpha}
\newcommand{\featmapmixmatrix}{\BA}

\newcommand{\featmapmixvecGroup}{\iter{\featmapmixvec}{\selectGroup}}
\newcommand{\featmapmixvecGroupTop}{\featmapmixvec^{(\selectGroup)\top}}
\newcommand{\featmapmixvecGroupSelect}{\featmapmixvecGroup_{\selectOutchannel}}
\newcommand{\featmapmixvecGroupTopSelect}{\featmapmixvecGroupTop_{\selectOutchannel}}
\newcommand{\featmapmixmatrixGroup}{\iter{\featmapmixmatrix}{\selectGroup}}

\newcommand{\inpimg}{\RMX}
\newcommand{\outimg}{\RMY}
\newcommand{\compleximg}{\RMZ}
\newcommand{\weightimg}{\RMV}
\newcommand{\complexWeightimg}{\RMW}
\newcommand{\dtimg}{\RMD}
\newcommand{\actimg}{\RMA}
\newcommand{\inpmultimg}{\BBX}
\newcommand{\outmultimg}{\BBY}

\newcommand{\weightmultimg}{\BBV}

\newcommand{\dtmultimg}{\BBD}

\newcommand{\colormiximg}{\inpimg^{\idxlum}}
\newcommand{\rmaximg}{\outimg^{\idxmax}}
\newcommand{\cmodimg}{\outimg^{\idxmod}}
\newcommand{\actRmaximg}{\actimg^{\idxmax}}
\newcommand{\actCmodimg}{\actimg^{\idxmod}}

\newcommand{\avgWeightimg}{\widetilde\weightimg}
\newcommand{\complexAvgWeightimg}{\widetilde\complexWeightimg}

\newcommand{\complexWeightimgDt}{\complexWeightimg^{\idxdtcwpt}}
\newcommand{\complexWeightimgDtSelect}{\complexWeightimgDt_{\selectDtchannel'}}

\newcommand{\dtimgSelect}{\dtimg_{\selectDtchannel}}

\newcommand{\dtimgPermGroup}{\iter{\dtimg}{\selectGroup}}

\newcommand{\dtimgPermGroupSelect}{\dtimgPermGroup_{\selectInchannel}}
\newcommand{\dtmultimgPerm}{\dtmultimg'}
\newcommand{\dtmultimgPermGroup}{\iter{\dtmultimg}{\selectGroup}}

\newcommand{\complexWeightimgDepth}{\iter{\complexWeightimg}{\depth}}
\newcommand{\complexWeightimgDepthSelect}{\complexWeightimgDepth_{\selectDtchannel}}

\newcommand{\inpimgSelect}{\inpimg_\selectInchannel}
\newcommand{\actimgSelect}{\actimg_\selectOutchannel}

\newcommand{\outimgSelect}{\outimg_\selectOutchannel}
\newcommand{\compleximgSelect}{\compleximg_\selectOutchannel}

\newcommand{\outimgGroup}{\iter{\outimg}{\selectGroup}}
\newcommand{\outimgGroupSelect}{\outimgGroup_\selectOutchannel}

\newcommand{\outmultimgHigh}{\outmultimg^{\idxhigh}}
\newcommand{\outmultimgGroup}{\iter{\outmultimg}{\selectGroup}}

\newcommand{\weightimgSelect}{\weightimg_{\selectOutchannel\selectInchannel}}
\newcommand{\complexWeightimgSelect}{\complexWeightimg_{\selectOutchannel\selectInchannel}}

\newcommand{\avgWeightimgSelect}{\avgWeightimg_\selectOutchannel}
\newcommand{\complexAvgWeightimgSelect}{\complexAvgWeightimg_\selectOutchannel}

\newcommand{\rmaximgSelect}{\rmaximg_\selectOutchannel}
\newcommand{\cmodimgSelect}{\cmodimg_\selectOutchannel}

\newcommand{\actRmaximgSelect}{\actRmaximg_\selectOutchannel}
\newcommand{\actCmodimgSelect}{\actCmodimg_\selectOutchannel}

\newcommand{\empiricalmean}{\Expval}
\newcommand{\empiricalvar}{\Var}

\newcommand{\empiricalmeanSub}{\empiricalmean_\sub}
\newcommand{\empiricalmeanTwosub}{\empiricalmean_{2\sub}}
\newcommand{\empiricalvarSub}{\empiricalvar_\sub}

\newcommand{\selectInchannelInrange}{\selectInchannel \in \setof{\inchannels}}
\newcommand{\selectInchannelInrangeGroup}{\selectInchannel \in \setof{\inchannelsGroup}}
\newcommand{\selectInchannelInRGB}{\selectInchannel \in \setof{3}}
\newcommand{\selectOutchannelInrange}{\selectOutchannel \in \setof{\outchannels}}
\newcommand{\selectOutchannelInrangeGroup}{\selectOutchannel \in \setof{\outchannelsGroup}}

\newcommand{\selectGroupInrange}{\selectGroup \in \setof{\ngroups}}
\newcommand{\selectDtchannelInrange}{\selectDtchannel \in \setof{\dtchannels}}

\newcommand{\maxVectorindex}{\max_{\norminfty{\vectorindexbis} \leq 1}}

\newcommand{\sumof}[2]{\sum_{#1 = 1}^{#2}}

\newcommand{\sumoverVectorindices}{\sum_{\vectorindex \in \mathZ^2}}
\newcommand{\sumoverVectorindicesbis}{\sum_{\vectorindexbis \in \mathZ^2}}
\newcommand{\sumoverInchannels}{\sumof{\selectInchannel}{\inchannels}}

\newcommand{\sumoverInchannelsRGB}{\sumof{\selectInchannel}{3}}

\newcommand{\objfun}{\ML}

\newcommand{\hilberttransf}{\MH}

\newcommand{\objfunbis}{\ML_0}

\newcommand{\inpimgStarWeightimg}{\inpimg \star \weightimg}

\begin{document}

\title{From CNNs to Shift-Invariant Twin Models Based on Complex Wavelets
\thanks{This work has been partially supported by the LabEx PERSYVAL-Lab (ANR-11-LABX-0025-01) funded by the French program Investissement d’avenir, as well as the ANR grant MIAI (ANR-19-P3IA-0003).
Most of the computations presented in this paper were performed using the GRICAD infrastructure
(\url{https://gricad.univ-grenoble-alpes.fr}),
which is supported by Grenoble research communities.}
}

\author{\IEEEauthorblockN{Hubert Leterme$^{\ast\dagger}$, Kévin Polisano$^\ddagger$, Valérie Perrier$^\ddagger$, and Karteek Alahari$^\dagger$}
\IEEEauthorblockA{$^\ast$Normandie Univ, UNICAEN, ENSICAEN, CNRS, GREYC, 14000 Caen, France \\
$^\dagger$Univ.\@ Grenoble Alpes, CNRS, Inria, Grenoble INP, LJK, 38000 Grenoble, France \\
$^\ddagger$Univ.\@ Grenoble Alpes, CNRS, Grenoble INP, LJK, 38000 Grenoble, France \\
E-mail: hubert.leterme@unicaen.fr}}

\maketitle

\begin{abstract}
    We propose a novel method to increase shift invariance and prediction accuracy in convolutional neural networks. Specifically, we replace the first-layer combination ``real-valued convolutions → max pooling'' (RMax) by ``complex-valued convolutions → modulus'' (CMod), which is stable to translations, or shifts. To justify our approach, we claim that CMod and RMax produce comparable outputs when the convolution kernel is band-pass and oriented (Gabor-like filter). In this context, CMod can therefore be considered as a stable alternative to RMax.
    To enforce this property, we constrain the convolution kernels to adopt such a Gabor-like structure. The corresponding architecture is called mathematical twin, because it employs a well-defined mathematical operator to mimic the behavior of the original, freely-trained model. Our approach achieves superior accuracy on ImageNet and CIFAR-10 classification tasks, compared to prior methods based on low-pass filtering. Arguably, our approach's emphasis on retaining high-frequency details contributes to a better balance between shift invariance and information preservation, resulting in improved performance. Furthermore, it has a lower computational cost and memory footprint than concurrent work, making it a promising solution for practical implementation.
\end{abstract}

\begin{IEEEkeywords}
    deep learning, image processing, shift invariance, max pooling, dual-tree complex wavelet packet transform, aliasing
\end{IEEEkeywords}

\section{Introduction}
\label{sec:introduction}

Over the past decade, some progress has been made on understanding the strengths and limitations of convolutional neural networks (CNNs) for computer vision \cite{LeCun2015,Wiatowski2018}. 
The ability of CNNs to embed input images into a feature space with linearly separable decision regions is a key factor to achieve high classification accuracy.
An important property to reach this linear separability is the ability to discard or minimize non-discriminative image components. In particular, feature vectors are expected to be stable with respect to translations \cite{Wiatowski2018}. However, subsampling operations, typically found in convolution and pooling layers, are an important source of instability---a phenomenon known as aliasing \cite{Azulay2019}.
A few approaches have attempted to address this issue.

\paragraph*{Blurpooled CNNs}
Zhang \cite{Zhang2019} proposed to apply a low-pass \emph{blurring} filter before each subsampling operation in CNNs. Specifically,
\begin{enumerate*}
    \item max pooling layers ($\maxlayer \to \subsamplayer$)\footnote{Sub and Conv stand for ``subsampling'' and ``convolution,'' respectively.} are replaced by max-blur pooling ($\maxlayer \to \blurlayer \to \subsamplayer$);
    \item convolution layers followed by ReLU ($\convlayer \to \subsamplayer \to \relulayer$) are blurred before subsampling ($\convlayer \to \relulayer \to \blurlayer \to \subsamplayer$).%
    \footnote{
        ReLU is computed before blurring; otherwise the network would simply perform on low-resolution images.
    }
\end{enumerate*}
The combination $\blurlayer \to \subsamplayer$ is referred to as \emph{blur pooling}. This approach follows a well-known practice called \emph{antialiasing}, which involves low-pass filtering a high-frequency signal before subsampling, in order to avoid artifacts in reconstruction.
Their approach improved the shift invariance as well as the accuracy of CNNs trained on ImageNet and CIFAR-10 datasets. However, this was achieved with a significant loss of information.

A question then arises: is it possible to design a non-destructive method, and if so, does it further improve accuracy? In a more recent work, Zou \etal \cite{Zou2023} tackled this question through an adaptive antialiasing approach, called \emph{adaptive blur pooling}.
Albeit achieving higher prediction accuracy, 
adaptive blur pooling requires additional memory, computational resources, and trainable parameters.

\paragraph*{Proposed Approach}
In this paper, we propose an alternative approach based on complex-valued convolutions, extracting high-frequency features that are stable to translations.
We observed improved accuracy for ImageNet and CIFAR-10 classification, compared to the two antialiasing methods based on blur pooling \cite{Zhang2019,Zou2023}.
Furthermore, our approach offers significant advantages in terms of computational efficiency and memory usage, and does not induce any additional training, unlike adaptive blur pooling.

Our proposed method replaces the first layers of a CNN: $\convlayer \to \subsamplayer \to \biaslayer \to \relulayer \to \maxpoollayer$, which can provably be rewritten as
\begin{equation}
   \convlayer \to \subsamplayer \to \maxpoollayer \to \biaslayer \to \relulayer,
\label{eq:rmaxmodel}
\end{equation}
by the following combination:
\begin{equation}
   \complexConvlayer \to \subsamplayer \to \moduluslayer \to \biaslayer \to \relulayer,
\label{eq:cmodmodel}
\end{equation}
where \complexConv denotes a convolution operator with a complex-valued kernel, whose real and imaginary parts approximately form a 2D Hilbert transform pair \cite{Havlicek1997}. From \eqref{eq:rmaxmodel} and \eqref{eq:cmodmodel}, we introduce the two following operators:
\begin{align}
   \rmaxlayer \;&:\; \convlayer \to \subsamplayer \to \maxpoollayer;
\label{eq:rmaxlayer} \\
   \cmodlayer \;&:\; \complexConvlayer \to \subsamplayer \to \moduluslayer.
\label{eq:cmodlayer}
\end{align}

Our method is motivated by the following theoretical claim. In a recent preprint \cite{Leterme2023}, we proved that
\begin{enumerate*}
    \item \cmod is nearly invariant to translations, if the convolution kernel is band-pass and clearly oriented;
    \item \rmax and \cmod produce comparable outputs, except for some filter frequencies regularly scattered across the Fourier domain.
\end{enumerate*}
We then combined these two properties to establish a stability metric for \rmax as a function of the convolution kernel's frequency vector.
This work was essentially theoretical, with limited experiments conducted on a deterministic model solely based on the dual-tree complex wavelet packet transform (\dtcwpt). However, it lacked applications to tasks such as image classification.
Building upon this theoretical study, in this paper, we consider the \cmod operator as a proxy for \rmax, extracting comparable, yet more stable features.

In compliance with the theory, the \rmax-\cmod substitution is only applied to the output channels associated with oriented band-pass filters, referred to as \emph{Gabor-like kernels}. This kind of structure is known to arise spontaneously in the first layer of CNNs trained on image datasets such as ImageNet \cite{Yosinski2014}. In this paper, we enforce this property by applying additional constraints to the original model. Specifically, a predefined number of convolution kernels are guided to adopt Gabor-like structures, instead of letting the network learn them from scratch. For this purpose, we rely on the dual-tree complex wavelet packet transform (\dtcwpt) \cite{Bayram2008}. Throughout the paper, we refer to this constrained model as a \emph{mathematical twin}, because it employs a well-defined mathematical operator to mimic the behavior of the original model. In this context, replacing \rmax by \cmod is straightforward, since the complex-valued filters are provided by \dtcwpt.

\paragraph*{Other Related Work}
Chaman and Dokmanic \cite{Chaman2021} reached perfect shift invariance by using an adaptive, input-dependent subsampling grid, whereas previous models rely on fixed grids. 
Although this method satisfied shift invariance for integer-pixel translations, it did not address the problem of shift instability for fractional-pixel translations, and therefore falls outside the scope of this paper.

Another aspect of shift invariance in CNNs is related to boundary effects. The fact that CNNs can encode the absolute position of an object in the image by exploiting boundary effects was discovered independently by Islam \etal \cite{Islam2020}, and Kayhan and Gemert \cite{Kayhan2020}.
This phenomenon is left outside the scope of our paper.
Finally, \cite{Biscione2021,Kvinge2022} studied the impact of pretraining on shift invariance and generalizability to out-of-distribution data, without modifying the network architecture.

\section{Proposed Approach}
\label{sec:proposedmodels}

We first describe the general principles of our approach based on complex convolutions. We then present the mathematical twin based on \dtcwpt, and explain how our method has been benchmarked against blur-pooling-based antialiased models.

We represent feature maps with straight capital letters: $\inpimg \in \twodseqs$, where $\twodseqs$ denotes the space of square-summable 2D sequences. Indexing is denoted by square brackets: for any 2D index $\vectorindex \in \mathZ^2$, $\inpimg[\vectorindex] \in \mathR$ or $\mathC$. The cross-correlation between $\inpimg$ and $\weightimg \in \twodseqs$ is defined by $(\inpimgStarWeightimg)[\vectorindex] := \sumoverVectorindicesbis \inpimg[\vectorindex + \vectorindexbis] \, \weightimg[\vectorindexbis]$. The down arrow refers to subsampling: for any $\sub \in \mathN^*$, $(\inpimg \downarrow \sub)[\vectorindex] := \inpimg[\sub\vectorindex]$.

\subsection{Standard Architectures}

A convolution layer with $\inchannels$ input channels, $\outchannels$ output channels and subsampling factor $\sub \in \mathN \setminus \{0\}$ is parameterized by a weight tensor $\weightmultimg := (\weightimgSelect)_{\selectOutchannelInrange,\, \selectInchannelInrange} \in \twodseqs^{\outchannels \times \inchannels}$. For any multichannel input $\inpmultimg := (\inpimgSelect)_{\selectInchannelInrange} \in \twodseqs^\inchannels$, the corresponding output $\outmultimg := (\outimgSelect)_{\selectOutchannelInrange} \in \twodseqs^\outchannels$ is defined such that, for any output channel $\selectOutchannelInrange$,
\begin{equation}
    \outimgSelect := \sumoverInchannels (\inpimgSelect \star \weightimgSelect) \downarrow \sub.
\label{eq:conv}
\end{equation}
For instance, in AlexNet and ResNet, $\inchannels = 3$ (RGB input images), $\outchannels = 64$, and $\sub = 4$ and $2$, respectively. Next, a bias $\biasvec := (\biasval_1,\, \cdots,\, \biasval_\outchannels)^\top \in \mathR^\outchannels$ is applied to $\outmultimg$, which is then transformed through nonlinear ReLU and max pooling operators. The activated outputs satisfy
\begin{equation}
	\actRmaximgSelect := \maxpool \left(\relu(\outimgSelect + \biasvalSelect)\right),
\label{eq:convrelumaxpool}
\end{equation}
where we have defined, for any $\outimg \in \twodseqs$ and any $\vectorindex \in \mathZ^2$,
\begin{align}
    \relu(\outimg)[\vectorindex]
        &:= \max(0,\, \outimg[\vectorindex]);
\label{eq:relu} \\
    \maxpool(\outimg)[\vectorindex]
        &:= \maxVectorindex \outimg[2\vectorindex + \vectorindexbis].
\label{eq:maxpool}
\end{align}

\subsection{Core Principle of our Approach}
\label{subsec:proposedmodels_principle}

We consider the first convolution layer of a CNN, as described in \eqref{eq:conv}. As widely discussed in the literature \cite{Yosinski2014}, after training with ImageNet, a certain number of convolution kernels $\weightimgSelect$ spontaneously take the appearance of oriented waveforms with well-defined frequency and orientation (Gabor-like kernels). A visual representation of trained convolution kernels is provided in \cref{fig:convkernels_wavealexnet}.
\begin{figure*}
    \centering
    \begin{subfigure}{.495\textwidth}
        \includegraphics[width=0.49\textwidth]{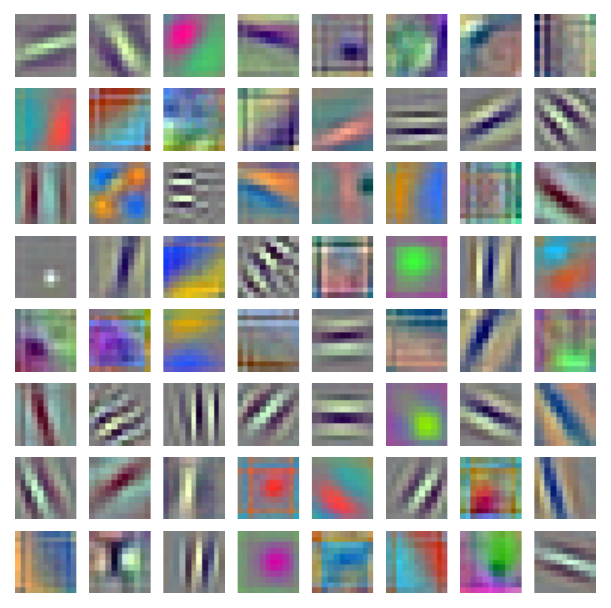}
        \includegraphics[width=0.49\textwidth]{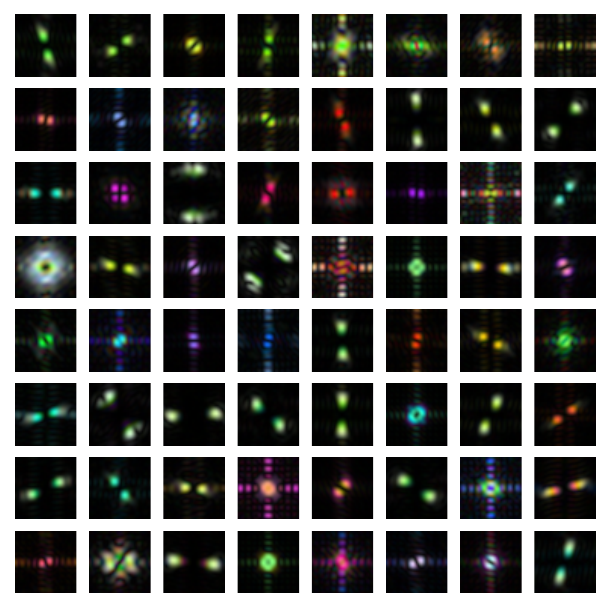}
        \caption{Standard AlexNet}
    \label{subfig:convkernels_ai}
    \end{subfigure}
    \begin{subfigure}{.495\textwidth}
        \includegraphics[width=0.49\textwidth]{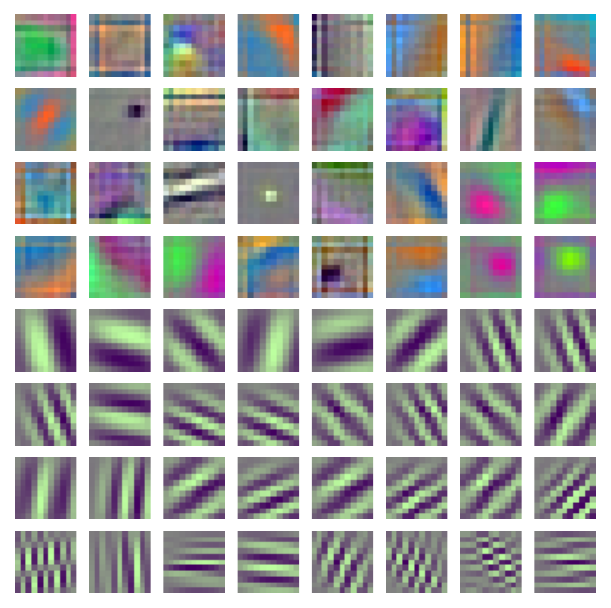}
        \includegraphics[width=0.49\textwidth]{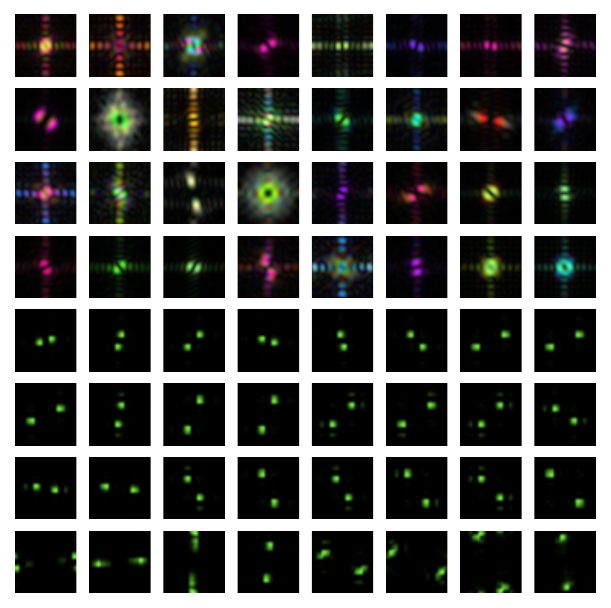}
        \caption{\wavealexnet (\dtcwpt-based twin)}
    \label{subfig:convkernels_awyi}
    \end{subfigure}
    \begin{subfigure}{.495\textwidth}
        \includegraphics[width=0.49\textwidth]{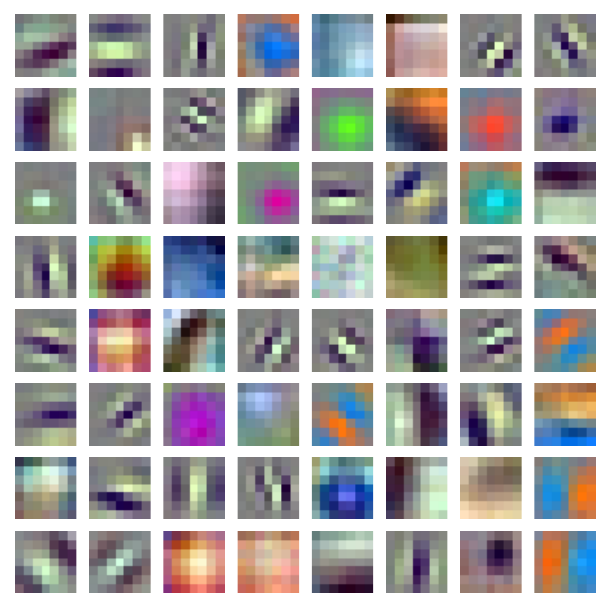}
        \includegraphics[width=0.49\textwidth]{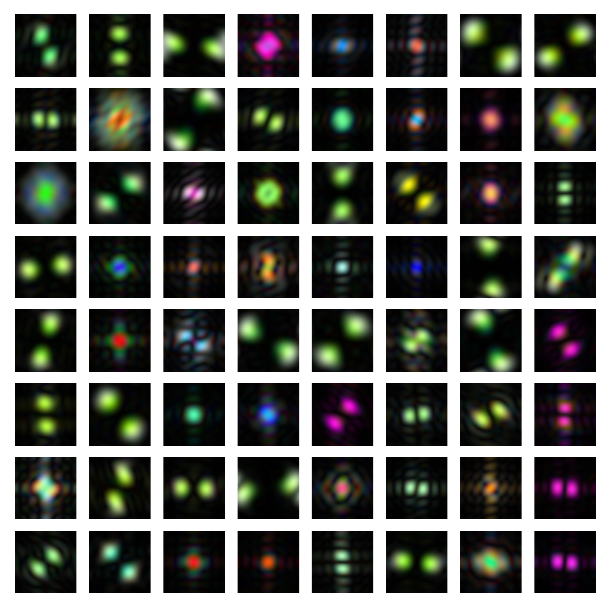}
        \caption{Standard ResNet-34}
    \label{subfig:convkernels_ri}
    \end{subfigure}
    \begin{subfigure}{.495\textwidth}
        \includegraphics[width=0.49\textwidth]{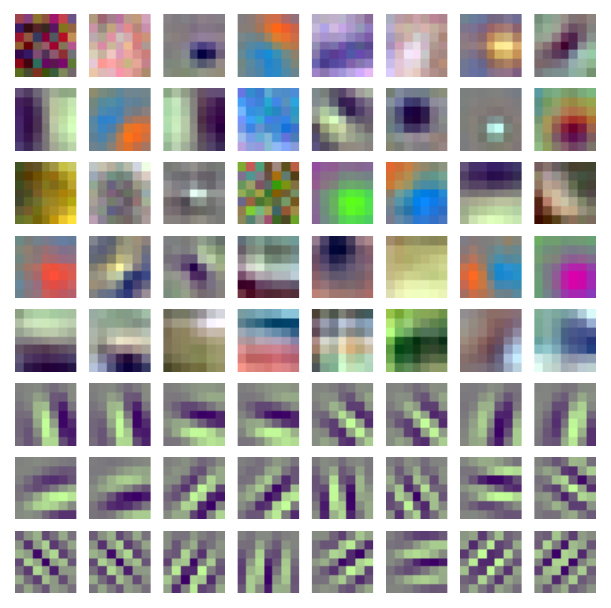}
        \includegraphics[width=0.49\textwidth]{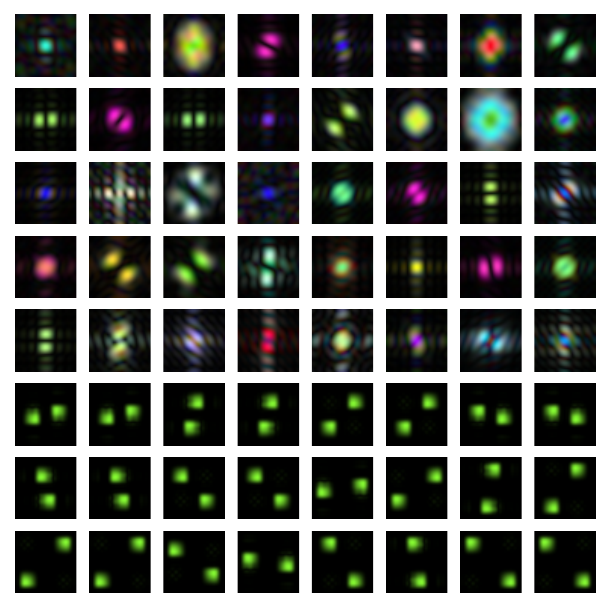}
        \caption{\waveresnet-34 (\dtcwpt-based twin)}
    \label{subfig:convkernels_rwyi}
    \end{subfigure}
    
    \caption{
        Convolution kernels $\weightmultimg \in \twodseqs^{64 \times 3}$ for the models based on AlexNet and ResNet-34, after training with ImageNet.
        Each image represents a 3D filter $(\weightimgSelect)_{\selectInchannelInRGB}$, for any output channel $\selectOutchannel \in \setof{64}$. For our \dtcwpt-based twin architecture (\cref{subfig:convkernels_awyi,subfig:convkernels_rwyi}), the $\outchannelsLow := 32$ or $40$ first kernels are freely-trained, whereas the remaining $\outchannelsHigh := 32$ or $24$ kernels are constrained to be monochrome, band-pass and oriented.
        Left: representation in the spatial domain; right: corresponding power spectra.
    }
    \vspace{-0pt}
    \label{fig:convkernels_wavealexnet}
\end{figure*}
In the present paper, we refer to these specific output channels $\selectOutchannel \in \gaborchannels \subset \setof{\outchannels}$ as \emph{Gabor channels}. The main idea is to substitute, for any $\selectOutchannel \in \gaborchannels$, \rmax by \cmod, as explained hereafter. Following \eqref{eq:rmaxmodel}, expression \eqref{eq:convrelumaxpool} can be rewritten
\begin{equation}
    \actRmaximgSelect = \relu\bigl(
        \rmaximgSelect + \biasvalSelect
    \bigr),
\label{eq:convmaxpoolrelu}
\end{equation}
where $\rmaximgSelect$ is the output of an \rmax operator as introduced in \eqref{eq:rmaxlayer}. More formally,
\begin{equation}
	\rmaximgSelect := \maxpool \left(
        \sumoverInchannels (\inpimgSelect \star \weightimgSelect) \downarrow \sub
    \right).
\label{eq:rmaxoutput}
\end{equation}
Then, following \eqref{eq:cmodmodel}, the \rmax-\cmod substitution yields
\begin{equation}
    \actCmodimgSelect = \relu\bigl(
        \cmodimgSelect + \biasvalSelect
    \bigr),
\label{eq:convmodulusrelu}
\end{equation}
where $\cmodimgSelect$ is the output of a \cmod operator \eqref{eq:cmodlayer}, satisfying
\begin{equation}
	\cmodimgSelect := \left|
        \sumoverInchannels (\inpimgSelect \star \complexWeightimgSelect) \downarrow (2\sub)
    \right|.
\label{eq:cmodoutput}
\end{equation}
In the above expression, $\complexWeightimgSelect$ is a complex-valued analytic kernel defined as $\complexWeightimgSelect := \weightimgSelect + i \hilberttransf(\weightimgSelect)$,
where $\hilberttransf$ denotes the two-dimensional \emph{Hilbert transform} as introduced by Havlicek \etal \cite{Havlicek1997}.
The Hilbert transform is designed such that the Fourier transform of $\complexWeightimgSelect$ is entirely supported in the half-plane of nonnegative \xcoord-values. Therefore, since $\weightimgSelect$ has a well-defined frequency and orientation, the energy of $\complexWeightimgSelect$ is concentrated within a small window in the Fourier domain. Due to this property, the modulus operator provides a smooth envelope for complex-valued cross-correlations with $\complexWeightimgSelect$ \cite{Kingsbury1998}. This leads to the output $\cmodimgSelect$ \eqref{eq:cmodoutput} being nearly invariant to translations.
Additionally, the subsampling factor in \eqref{eq:cmodoutput} is twice that in \eqref{eq:rmaxoutput}, to account for the factor-$2$ subsampling achieved through max pooling \eqref{eq:maxpool}.

\subsection{Wavelet-Based Twin Models (\wavecnns)}
\label{subsec:proposedmodels_baseline}

As explained in \cref{subsec:proposedmodels_principle}, introducing an imaginary part to the Gabor-like convolution kernels improves shift invariance. Our method therefore restricts to the Gabor channels $\selectOutchannel \in \gaborchannels \subset \setof{\outchannels}$. However, $\gaborchannels$ is unknown a priori: for a given output channel $\selectOutchannelInrange$, whether $\weightimgSelect$ will become band-pass and oriented after training is unpredictable. Thus, we need a way to automatically separate the set $\gaborchannels$ of Gabor channels from the set of remaining channels, denoted by $\freechannels := \setof{\outchannels} \setminus \gaborchannels$. To this end, we built ``mathematical twins'' of standard CNNs, based on the dual-tree wavelet packet transform (\dtcwpt). These models, which we call \wavecnns, reproduce the behavior of freely-trained architectures with a higher degree of control and fewer trainable parameters.
In short, the two groups of output channels are organized such that $\freechannels = \setof{\outchannelsLow}$ and $\gaborchannels = \range{(\outchannelsLow + 1)}{\outchannels}$.
The first $\outchannelsLow$ channels, which are outside the scope of our approach, remain freely-trained as in the standard architecture. The remaining $\outchannelsHigh := 1 - \outchannelsLow$ channels are constrained to adopt a Gabor-like structure with deterministic frequencies and orientations, through the implementation of \dtcwpt.
Using the principles introduced in \cref{subsec:proposedmodels_principle}, we then replace \rmax \eqref{eq:rmaxoutput} by \cmod \eqref{eq:cmodoutput} for all Gabor channels $\selectOutchannel \in \gaborchannels$. The corresponding models are referred to as \cwavecnns.
A detailed description of \wavecnns and \cwavecnns is provided in Appendix~\ref{sec:appendix_wcnn_genarch},
together with schematic representations.

\subsection{\wavecnns with Blur Pooling}
\label{subsec:proposedmodels_antialiasing}

We benchmark our approach against the antialiasing methods proposed by Zhang \cite{Zhang2019} and Zou \etal \cite{Zou2023}. To this end, we first consider a \wavecnn antialiased with static or adaptive blur pooling, respectively referred to as \blurwavecnn and \adablurwavecnn. Then, we substitute the blurpooled Gabor channels with our own \cmod-based approach. The corresponding models are respectively referred to as \cblurwavecnn and \cadablurwavecnn. A schematic representation of \blurwavealexnet and \cblurwavealexnet can be found in \cref{fig:models}.

\section{Experiments}
\label{sec:experiments}

To ensure reproducibility, we have released the code associated with our study on GitHub.%
\footnote{
    \url{https://github.com/hubert-leterme/wcnn}
}

\subsection{Experiment Details}
\label{subsec:experiments_details}

\paragraph*{ImageNet}
We built our \wavecnn and \cwavecnn twin models based on AlexNet \cite{Krizhevsky2017} and ResNet-34 \cite{He2016}. The hyperparameter $\outchannelsLow$ was manually chosen based on empirical observations ($32$ for AlexNet and $40$ for ResNet-34). Besides, \dtcwpt decompositions were performed with Q-shift orthogonal filters of length $10$ as introduced by Kingsbury \cite{Kingsbury2003}.
More details can be found in Appendix~\ref{sec:appendix_impl_details}.

Zhang's static blur pooling approach has been tested on both AlexNet and ResNet, whereas Zou \etal's adaptive approach has only been tested on ResNet. The latter was indeed not implemented on AlexNet in the original paper, and we were unable to adapt it to this architecture.

Our models were trained on the ImageNet ILSVRC2012 dataset \cite{Russakovsky2015}, following the standard procedure provided by the PyTorch library \cite{Paszke2017}.%
\footnote{
    PyTorch ``examples'' repository available at
    \newblock \url{https://github.com/pytorch/examples/tree/main/imagenet}
} Moreover, we set aside $100$K images from the training set---$100$ per class---in order to compute the top-$1$ error rate after each training epoch (``validation set'').

\paragraph*{CIFAR-10}
We also trained ResNet-18- and ResNet-34-based models on the CIFAR-10 dataset. Training was performed on $300$ epochs, with an initial learning rate set to $0.1$, decreased by a factor of $10$ every $100$ epochs. We set aside $5\,000$ images out of $50$K to compute accuracy during the training phase.

\subsection{Evaluation Metrics}
\label{subsec:experiments_metrics}
\paragraph*{Classification Accuracy}
Classification accuracy was computed on the ImageNet test set ($50$K images). We followed the \emph{ten-crops} procedure~\cite{Krizhevsky2017}: predictions are made over $10$ patches extracted from each input image, and the softmax outputs are averaged to get the overall prediction. We also considered center crops of size $224$ for \emph{one-crop} evaluation. In both cases, we used top-1-5 error rates. For CIFAR-10 evaluation ($10$K images in the test set), we measured the top-1 error rate with one- and ten-crops.

\paragraph*{Measuring Shift Invariance}
For each image in the ImageNet evaluation set, we extracted several patches of size $224$, each of which being shifted by $0.5$ pixel along a given axis. We then compared their outputs in order to measure the model's robustness to shifts. This was done by computing the Kullback-Leibler (KL) divergence between output vectors---which, under certain hypotheses, can be interpreted as probability distributions \cite[pp.~205-206]{Bishop2014}. This metric is intended for visual representation (see \cref{fig:valcurves_shifts}).

In addition, we measured the mean flip rate (mFR) between predictions \cite{Hendrycks2019}, as done by Zhang \cite{Zhang2019} in its blurpooled models. For each direction (vertical, horizontal and diagonal), we measured the mean frequency upon which two shifted input images yield different top-1 predictions, for shift distances varying from $1$ to $8$ pixels. We then normalized the results with respect to AlexNet's mFR, and averaged over the three directions. This metric is also referred to as \emph{consistency}.

We repeated the procedure for the models trained on CIFAR-10. This time, we extracted patches of size $32 \times 32$ from the evaluation set, and computed mFR for shifts varying from $1$ to $4$ pixels. Normalization was performed with respect to ResNet-18's mFR.

\subsection{Results and Discussion}
\label{subsec:experiments_results}

\begin{figure}
    \centering
    \includegraphics[width=\columnwidth, height=0.21\textheight]{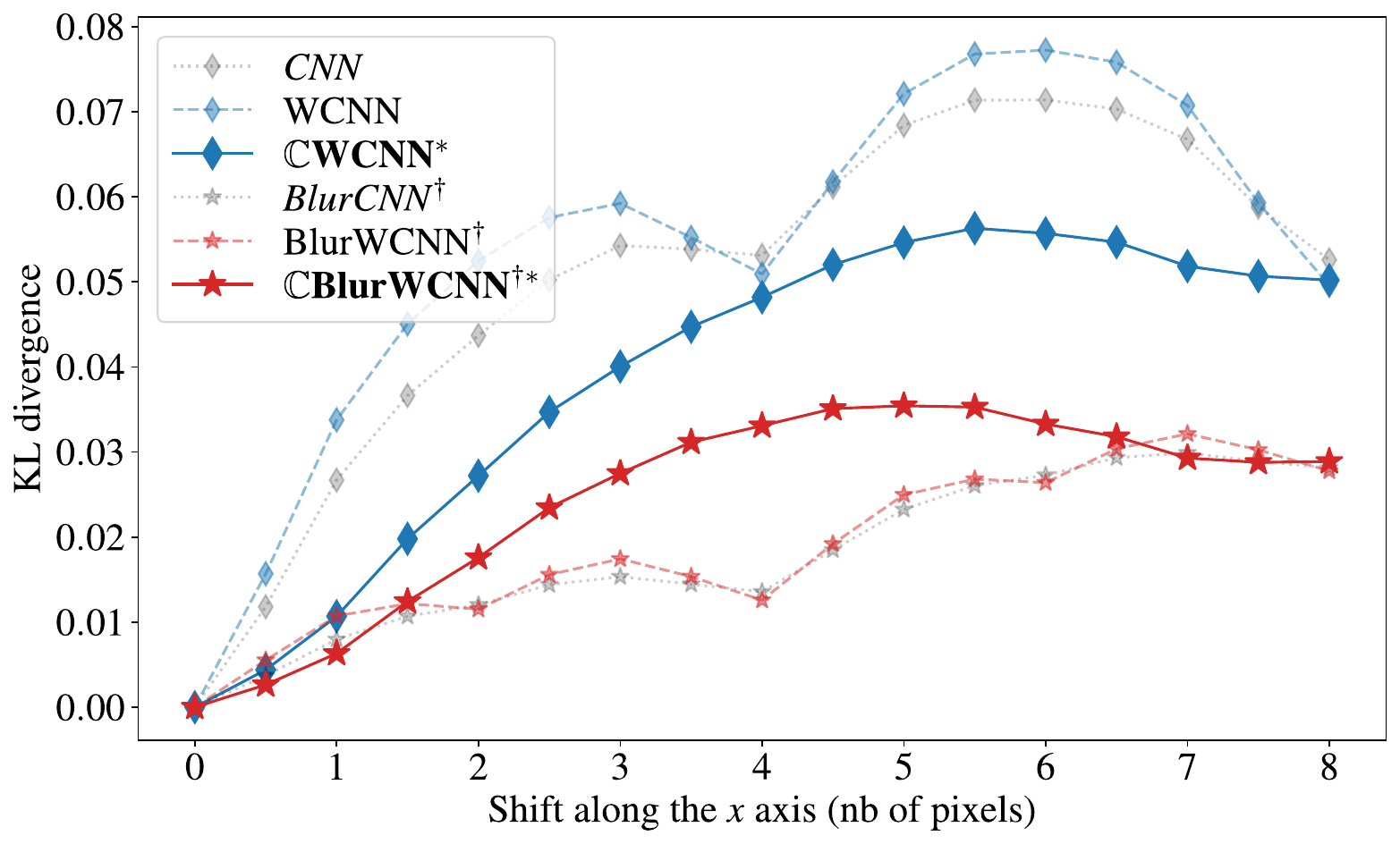}
    \vspace{-20pt}
    \caption{
        AlexNet-based models: mean KL divergence between the outputs of
        shifted images. Legend: $^\dagger$blur pooling; $^\ast$\cmod-based approach (ours).
    }
    \label{fig:valcurves_shifts}
    \vspace{-15pt}
\end{figure}

\begin{table}
    \renewcommand{\arraystretch}{1.}

    \caption{Evaluation metrics on ImageNet (\%): the lower the better}
    \vspace{-5pt}
    \centering\footnotesize
    \begin{tabular}{r|rr|rr|r}
        \hline
        \multicolumn{1}{c|}{\multirow{2}{*}{\textbf{Model}}} & \multicolumn{2}{c|}{\textbf{One-crop}} & \multicolumn{2}{c|}{\textbf{Ten-crops}} & \multicolumn{1}{c}{\textbf{Shifts}} \\

        & \multicolumn{1}{c}{top-1} & \multicolumn{1}{c|}{top-5} & \multicolumn{1}{c}{top-1} & \multicolumn{1}{c|}{top-5} & \multicolumn{1}{c}{mFR} \\

        \hline
        \multicolumn{1}{c|}{} & \multicolumn{5}{c}{\textbf{AlexNet}} \\
        \hline

        \textit{CNN} & \textit{45.3} & \textit{22.2} & \textit{41.3} & \textit{19.3} & \textit{100.0} \\

        \wavecnn & 44.9 & 21.8 & 40.8 & 19.0 & 101.4 \\

         {\cwavecnn}$^\ast$ & \textbf{44.3} & \textbf{21.3} & \textbf{40.2} & \textbf{18.5} & \textbf{88.0} \\

        \hdashline

        \textit{\blurcnn}$^\dagger$ & \textit{44.4} & \textit{21.6} & \textit{40.7} & \textit{18.7} & \textit{63.8} \\

        {\blurwavecnn}$^\dagger$ & 44.3 & 21.4 & 40.5 & 18.5 & \textbf{63.1} \\

        {\cblurwavecnn}$^{\dagger\ast}$ & \textbf{43.3} & \textbf{20.5} & \textbf{39.6} & \textbf{17.9} & 69.4 \\

        \hline
        \multicolumn{1}{c|}{} & \multicolumn{5}{c}{\textbf{ResNet-34}}  \\
        \hline

        \textit{CNN} & \textit{27.6} & \textit{9.2} & \textit{24.8} & \textit{7.7} & \textit{78.1} \\

        \wavecnn & 27.4 & 9.2 & 24.7 & 7.6 & 77.2 \\

        {\cwavecnn}$^\ast$ & \textbf{27.2} & \textbf{9.0} & \textbf{24.4} & \textbf{7.4} & \textbf{73.1} \\

        \hdashline

        \textit{\blurcnn}$^\dagger$ & \textit{26.7} & \textit{8.6} & \textit{24.0} & \textit{7.2} & \textit{61.2}  \\

        {\blurwavecnn}$^\dagger$ & 26.7 & 8.6 & 24.1 & 7.3 & 65.2 \\

        {\cblurwavecnn}$^{\dagger\ast}$ & \textbf{26.5} & \textbf{8.4} & \textbf{23.7} & \textbf{7.0} & \textbf{62.5} \\

        \hdashline

        \textit{\adablurcnn}$^{\ddagger}$ & \textit{26.1} & \textit{8.3} & \textit{23.5} & \textit{7.0} & \textit{60.8} \\

        {\adablurwavecnn}$^{\ddagger}$ & \textbf{26.0} & \textbf{8.2} & \textbf{23.6} & \textbf{6.9} & \textbf{62.1} \\

        {\cadablurwavecnn}$^{\ddagger\ast}$ & 26.1 & \textbf{8.2} & 23.7 & 7.0 & 63.1 \\

        \hline

        \multicolumn{6}{l}{$^\dagger$static and $^\ddagger$adaptive blur pooling; $^\ast$\cmod-based approach (ours)}
    \end{tabular}
	\label{table:results_imagenet}

    \vspace{5pt}

    \caption{Evaluation metrics on CIFAR-10 (\%): the lower the better}
    \vspace{-5pt}

    \begin{tabular}{r|rrr|rrr}
        \hline
        \multicolumn{1}{c|}{\multirow{2}{*}{\textbf{Model}}} & \multicolumn{3}{c|}{\textbf{ResNet-18}} & \multicolumn{3}{c}{\textbf{ResNet-34}} \\

        & \multicolumn{1}{c}{1crp} & \multicolumn{1}{c}{10crp} & \multicolumn{1}{c|}{shifts} & \multicolumn{1}{c}{1crp} & \multicolumn{1}{c}{10crps} & \multicolumn{1}{c}{shifts} \\

        \hline

        \textit{CNN} & \textit{14.9} & \textit{10.8} & \textit{100.0} & \textit{15.2} & \textit{10.9} & \textit{100.3} \\

        \wavecnn & 14.2 & 10.3 & 92.4 & 14.5 & 10.5 & 99.2 \\

        {\cwavecnn}$^\ast$ & \textbf{13.8} & \textbf{9.6} & \textbf{88.8} & \textbf{12.9} & \textbf{9.2} & \textbf{93.0} \\

        \hdashline

        \textit{\blurcnn}$^\dagger$ & \textit{14.2} & \textit{10.4} & \textit{87.7} & \textit{15.7} & \textit{11.6} & \textit{88.2} \\

        {\blurwavecnn}$^\dagger$ & 13.1 & 9.7 & \textbf{84.6} & 13.2 & 9.9 & 85.6 \\

        {\cblurwavecnn}$^{\dagger\ast}$ & \textbf{12.3} & \textbf{8.9} & 85.7 & \textbf{12.4} & \textbf{9.1} & \textbf{83.7} \\

        \hdashline

        \textit{\adablurcnn}$^{\ddagger}$ & \textit{14.6} & \textit{11.0} & \textit{90.9} & \textit{16.3} & \textit{12.8} & \textit{91.9} \\

        {\adablurwavecnn}$^{\ddagger}$ & 14.5 & 11.0 & 86.5 & 14.0 & 10.4 & 93.3 \\

        {\cadablurwavecnn}$^{\ddagger\ast}$ & \textbf{12.8} & \textbf{9.7} & \textbf{81.7} & \textbf{12.8} & \textbf{9.2} & \textbf{86.6} \\

        \hline

        \multicolumn{7}{l}{\emph{1crp} and \emph{10crp}: top-1 error rate using one- and ten-crops methods} \\

        \multicolumn{7}{l}{\emph{shifts}: mFR measuring consistency} \\

        \multicolumn{7}{l}{$^\dagger$static and $^\ddagger$adaptive blur pooling; $^\ast$\cmod-based approach (ours)}
    \end{tabular}
	\label{table:results_cifar}
    \vspace{-15pt}
\end{table}

\paragraph*{Validation and Test Accuracy}
Error rates of AlexNet- and ResNet-based architectures, computed on the test sets, are provided in \cref{table:results_imagenet} for ImageNet and \cref{table:results_cifar} for CIFAR-10.

When trained on ImageNet, our \cmod-based approach significantly outperforms the baselines for AlexNet: \cwavecnn vs \wavecnn, and \cblurwavecnn vs \blurwavecnn. 
Positive results are also obtained for ResNet-based models trained on ImageNet. However, adaptive blur pooling, when applied to the Gabor channels (\adablurwavecnn), yields similar or marginally higher accuracy than our approach (\cadablurwavecnn).
Nevertheless, our method is computationally more efficient, requires less memory (see ``Computational Resources'' below for more details), and does not demand additional training, unlike adaptive blur pooling.
On the other hand, when trained on CIFAR-10, our approach systematically yields the lowest error rates.

\paragraph*{Shift Invariance (KL Divergence)}
The mean KL divergence between the outputs of shifted images are plotted in \cref{fig:valcurves_shifts} for AlexNet trained on ImageNet. The mean flip rate for shifted inputs (consistency) is reported in \cref{table:results_imagenet} for ImageNet (AlexNet and ResNet-34) and \cref{table:results_cifar} for CIFAR-10 (ResNet-18 and 34).

In models without blur pooling (blue curves), the \rmax-\cmod substitution greatly reduces first-layer instabilities, resulting in a flattened curve and avoiding the ``bumps'' observed for non-stabilized models.
On the other hand, when applied to the blurpooled models (red curves), the \rmax-\cmod substitution actually tends to degrade shift invariance, as evidenced by the bell-shaped curve. Nevertheless, the corresponding classifier is significantly more accurate, as shown in \cref{table:results_imagenet}. This is not surprising, as our approach prioritizes the conservation of high-frequency details, which are important for classification. An extreme reduction of shift variance using a large blur pooling filter would indeed result in a significant loss of accuracy.
Therefore,
our work achieves a better tradeoff between shift invariance and information preservation.

To gain further insights into this phenomenon, we conducted experiments by varying the size of the blurring filters. \Cref{fig:tradeoff} shows the relationship between consistency and prediction accuracy on ImageNet (custom validation set), for AlexNet-based models with different blurring filter sizes ranging from $1$ (no blur pooling) to $7$ (heavy loss of high-frequency information).
Additional plots are provided in Appendix~\ref{sec:appendix_accuracy_vs_consistency},
for the test set as well as ResNet-based models. We find that a near-optimal trade-off is achieved when the filter size is set to $2$ or $3$. Furthermore, at equivalent consistency levels, \cblurwavecnn (our approach) outperforms \blurwavecnn in terms of accuracy.

\begin{figure}
    \centering
    \includegraphics[width=\columnwidth]{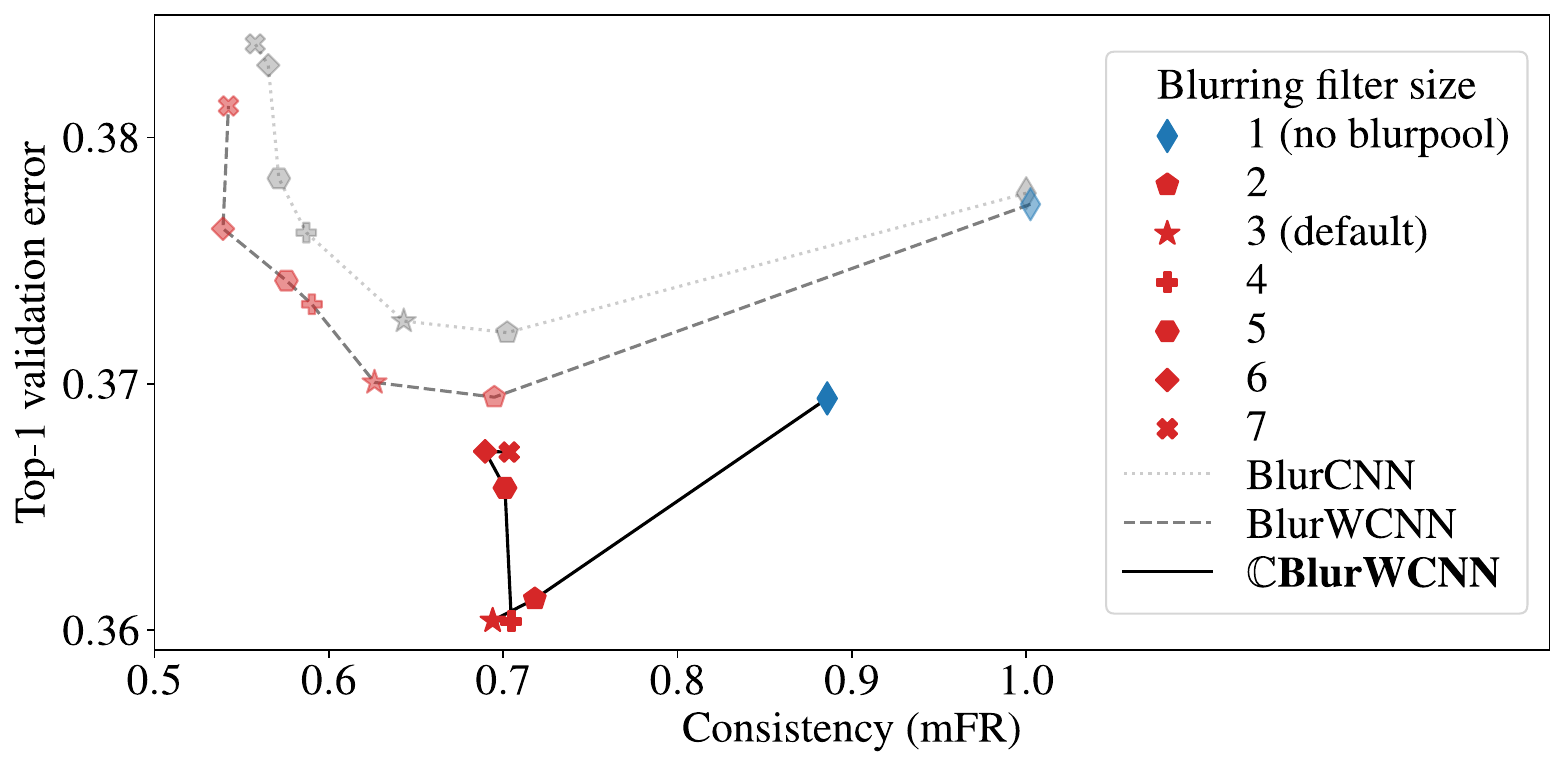}
    \vspace{-20pt}
    \caption{
        Classification accuracy (ten-crops) vs consistency, measuring the stability of predictions to small input shifts, for AlexNet-based models (the lower the better for both axes). For each of the three architectures, we increased the blurring filter size from $1$ (\ie, no blur pooling) to $7$. The blue diamonds (no blur pooling) and red stars (blur pooling with filters of size $3$) correspond to the models for which evaluation metrics have been reported in \cref{table:results_imagenet} (models trained after $90$ epochs).
    }
    \label{fig:tradeoff}
\end{figure}


As a side note, because shift invariance is desirable for a wide range of tasks and datasets, embedding this property into CNNs may improve generalizability and avoid overfitting. 

\paragraph*{Computational Resources}
\Cref{table:complexity} displays the computational resources and memory footprint required for each method, per Gabor channel. The values are normalized relative to non-stabilized AlexNet or ResNet. The metrics are, on the one hand, the FLOPs necessary for computing $\rmaximgSelect$ \eqref{eq:rmaxoutput} or $\cmodimgSelect$ \eqref{eq:cmodoutput}, and, on the other hand, the size of the intermediate and output tensors saved by PyTorch for the backward pass. More details are provided in Appendix~\ref{sec:appendix_computational_costs}.

\begin{table}
    \caption{Computational cost and memory footprint}
    \vspace{-3pt}
    \renewcommand{\arraystretch}{1.}
    \centering\footnotesize
    \begin{tabular}{r|rr|rr}
        \hline
        \multicolumn{1}{c|}{\multirow{2}{*}{\textbf{Method}}} & \multicolumn{2}{c|}{\textbf{Computational cost}} & \multicolumn{2}{c}{\textbf{Memory footprint}} \\
        & AlexNet & ResNet & AlexNet & ResNet \\
        \hline
        \textit{No antialiasing (ref)} & $\mathit{1.0}$ & $\mathit{1.0}$ & $\mathit{1.0}$ & $\mathit{1.0}$ \\
        BlurPool \cite{Zhang2019} & $4.0$ & $1.0$ & $4.7$ & $1.9$ \\
        ABlurPool \cite{Zou2023} & -- & $2.1$ & -- & $2.0$ \\
        \textbf{\cmod (ours)} & $\mathbf{0.5}$ & $\mathbf{0.5}$ & $\mathbf{0.6}$ & $\mathbf{0.4}$ \\
        \hline
    \end{tabular}
	\label{table:complexity}
    \vspace{-8pt}
\end{table}

The observed improvements are mainly due to the larger stride (\ie, subsampling factor) in the first layer, allowing for smaller intermediate feature maps.

\section{Conclusion}
The mathematical twins introduced in this paper serve as a proof of concept for our \cmod-based approach. However, its range of application extends well beyond \dtcwpt filters.
It is important to note that such initial layers play a critical role in CNNs by extracting low-level geometric features such as edges, corners or textures. Therefore, a specific attention is required for their design. In contrast, deeper layers are more focused on capturing high-level structures that conventional image processing tools are poorly suited for \cite{Oyallon2017a}.

Furthermore,
our approach
has potential for broader applicability beyond CNNs. There is a growing interest in using self-attention mechanisms in computer vision \cite{Dosovitskiy2021} to capture complex, long-range dependencies among image representations. Recent work on vision transformers has proposed using the first layers of a CNN as a ``convolutional token embedding'' \cite{Hassani2022,Wu2021,Yuan2021}, effectively reintroducing inductive biases to the architecture, such as locality and weight sharing. By applying our method to this embedding, we can potentially provide self-attention modules with shift-invariant inputs. This could be beneficial in improving the performance of vision transformers, especially when the amount of available data is limited.

\appendices

\section{Design of \wavecnns: General Architecture}
\label{sec:appendix_wcnn_genarch}

In this section, we provide complements to the description of the mathematical twin (\wavecnn) introduced in \cref{subsec:proposedmodels_baseline,subsec:proposedmodels_antialiasing}.

We assume, without loss of generality, that $\inchannels = 3$ (RGB input images).
The numbers $\outchannelsLow$ and $\outchannelsHigh$ of freely-trained and Gabor channels are empirically determined from the trained CNNs (see \cref{subfig:convkernels_ai,subfig:convkernels_ri}).
In a twin \wavecnn architecture, the two groups of output channels are organized such that $\freechannels = \setof{\outchannelsLow}$ and $\gaborchannels = \range{(\outchannelsLow + 1)}{\outchannels}$.
The first $\outchannelsLow$ channels, which are outside the scope of our approach, remain freely-trained, like in the standard architecture. Regarding the $\outchannelsHigh$ remaining channels (Gabor channels), the convolution kernels $\weightimgSelect$ with $\selectOutchannel \in \gaborchannels$ are constrained to satisfy the following requirements. First, all three RGB input channels are processed with the same filter, up to a multiplicative constant. More formally, there exists a \emph{luminance} weight vector $\colormixvec := (\colormixval_1,\, \colormixval_2,\, \colormixval_3)^\top$, with $\colormixvalSelect \in \zeroone$ and $\sumoverInchannelsRGB \colormixvalSelect = 1$, such that,
\begin{equation}
    \forall \selectInchannelInRGB,\, \weightimgSelect = \colormixvalSelect\avgWeightimgSelect,
\label{eq:monochrome}
\end{equation}
where $\avgWeightimgSelect := \sumoverInchannelsRGB \weightimgSelect$ denotes the mean kernel. Furthermore, $\avgWeightimgSelect$ must be band-pass and oriented (Gabor-like filter). The following paragraphs explain how these two constraints are implemented in our \wavecnn architecture.

\subsection{Monochrome Filters}
Expression \eqref{eq:monochrome} is actually a property of standard CNNs: the oriented band-pass RGB kernels generally appear monochrome (see kernel visualization of freely-trained CNNs in \cref{subfig:convkernels_ai,subfig:convkernels_ri}). In \wavecnns, this constraint is implemented with a trainable $\obo$ convolution layer \cite{Lin2014}, parameterized by $\colormixvec$, computing the following luminance image:
\begin{equation}
    \colormiximg := \sumoverInchannelsRGB \colormixvalSelect \inpimgSelect.
\label{eq:colormix}
\end{equation}
This constraint can be relaxed by authorizing a specific luminance vector $\colormixvec_\selectOutchannel$ for each Gabor channel $\selectOutchannel \in \gaborchannels$. Numerical experiments on such models are left for future work.

\subsection{Gabor-Like Kernels}
To guarantee the Gabor-like property on $\avgWeightimgSelect$, we implemented \dtcwpt, which is achieved through a series of subsampled convolutions.
The number of decomposition stages $\depth \in \mathN \setminus \{0\}$ was chosen such that $\sub = 2^{\depth - 1}$, where, as a reminder, $\sub$ denotes the subsampling factor as introduced in \eqref{eq:conv}. \dtcwpt generates a set of filters $\bigl(
    \complexWeightimgDtSelect
\bigr)_{\selectDtchannel' \in \setof{4 \times 4^\depth}}$, which tiles the Fourier domain $\mpipi^2$ into $4 \times 4^\depth$ overlapping square windows.
Their real and imaginary parts approximately form a 2D Hilbert transform pair. \Cref{fig:convkernels_example} illustrates such a convolution filter.

\begin{figure}
    \centering
    \input{inputs/convkernels_example_preprint.tex}
    \caption{
        (a), (b): Real and imaginary parts of a Gabor-like convolution kernel $\complexWeightimgSelect := \weightimgSelect + i\hilberttransf(\weightimgSelect)$, forming a 2D Hilbert transform pair. (c), (d): Power spectra (energy of the Fourier transform) of $\weightimgSelect$ and $\complexWeightimgSelect$, respectively.
    }
    \vspace{-0pt}
    \label{fig:convkernels_example}
\end{figure}

The \wavecnn architecture is designed such that, for any Gabor channel $\selectOutchannel \in \gaborchannels$, $\avgWeightimgSelect$ is the real part of one such filter:
\begin{equation}
    \exists \selectDtchannel' \in \setof{4 \times 4^\depth} : \avgWeightimgSelect = \Real\bigl(
        \complexWeightimgDtSelect
    \bigr).
\label{eq:gaborfilter}
\end{equation}
The output $\outimgSelect$ introduced in \eqref{eq:conv} then becomes
\begin{equation}
    \outimgSelect = \bigl(
        \colormiximg \star \avgWeightimgSelect
    \bigr) \downarrow 2^{\depth - 1}.
\label{eq:dtcwpt_real}
\end{equation}


To summarize, a \wavecnn substitutes the freely-trained convolution \eqref{eq:conv} with a combination of \eqref{eq:colormix} and \eqref{eq:dtcwpt_real}, for any Gabor output channels $\selectOutchannel \in \gaborchannels$. This combination is wrapped into a \emph{wavelet block}, also referred to as \waveblock. Technical details about its exact design are provided in \cref{sec:wcnn_techdetails}.
Note that the Fourier resolution of $\weightimgSelect$ increases with the subsampling factor $\sub$. This property is consistent with what is observed in freely-trained CNNs: in AlexNet, where $\sub = 4$, the Gabor-like filters are more localized in frequency (and less spatially localized) than in ResNet, where $\sub = 2$.

Visual representations of the kernels $\weightmultimg \in \twodseqs^{\outchannels \times \inchannels}$, with $\inchannels = 3$ and $\outchannels = 64$, for the \wavecnn architectures based on AlexNet and ResNet-34, referred to as \wavealexnet and \waveresnet-34, are provided in \cref{subfig:convkernels_awyi,subfig:convkernels_rwyi}, respectively.

\subsection{Stabilized \wavecnns}

\begin{figure*}
    \centering
    \input{inputs/models_alexnet}
    \caption{
        First layers of AlexNet and its variants, corresponding to a convolution layer followed by ReLU and max pooling \eqref{eq:rmaxmodel}. The models are framed according to the same colors and line styles as in \cref{fig:valcurves_shifts} (main paper). The green modules are the ones containing trainable parameters; the orange and purple modules represent static linear and nonlinear operators, respectively. The numbers between each module represent the depth (number of channels), height and width of each output.
        \cref{subfig:models_alexnet}: freely-trained models. Top: standard AlexNet. Bottom: Zhang's ``blurpooled'' AlexNet.
        \cref{subfig:models_wavealexnet}: mathematical twins (\wavealexnet) reproducing the behavior of standard (top) and blurpooled (bottom) AlexNet.
        The left side of each diagram corresponds to the $\outchannelsLow := 32$ freely-trained
        output channels, whereas the right side displays the $\outchannelsHigh := 32$ remaining channels, where freely-trained convolutions have been replaced by a wavelet block (\waveblock) as described in \cref{sec:appendix_wcnn_genarch}. 
        \cref{subfig:models_cwavealexnet}: \cmod-based \wavealexnet, where \waveblock has been replaced by \cwaveblock, and max pooling by a modulus. The bias and ReLU are placed after the modulus, following \eqref{eq:cmodmodel}. In the bottom models, we compare Zhang's antialiasing approach (\cref{subfig:models_wavealexnet}) with ours (\cref{subfig:models_cwavealexnet}) in the Gabor channels.
    }
    \vspace{-0pt}
\label{fig:models}
\end{figure*}

Using the principles presented in \cref{subsec:proposedmodels_principle} of the main paper, we replace \rmax \eqref{eq:rmaxoutput} by \cmod \eqref{eq:cmodoutput} for all Gabor channels $\selectOutchannel \in \gaborchannels$. In the corresponding model, referred to as \cwavecnn, the wavelet block is replaced by a \emph{complex wavelet block} (\cwaveblock), in which \eqref{eq:dtcwpt_real} becomes
\begin{equation}
    \compleximgSelect = \bigl(
        \colormiximg \star \complexAvgWeightimgSelect
    \bigr) \downarrow 2^{\depth},
\label{eq:dtcwpt_complex}
\end{equation}
where $\complexAvgWeightimgSelect$ is obtained by considering both real and imaginary parts of the \dtcwpt filter:
\begin{equation}
    \complexAvgWeightimgSelect := \complexWeightimgDtSelect,
\end{equation}
where $\selectDtchannel'$ has been introduced in \eqref{eq:gaborfilter}. Then, a modulus operator is applied to $\compleximgSelect$, which yields $\cmodimgSelect$ such as defined in \eqref{eq:cmodoutput}, with $\complexWeightimgSelect := \colormixvalSelect\complexAvgWeightimgSelect$ for any RGB channel $\selectInchannelInRGB$. Finally, we apply a bias and ReLU to $\cmodimgSelect$, following \eqref{eq:convmodulusrelu}.

A schematic representation of \wavealexnet and its stabilized version, referred to as \cwavealexnet, is provided in \cref{fig:models} (top part).
Following \cref{subsec:proposedmodels_antialiasing}, the \wavecnn and \cwavecnn architectures built upon blurpooled AlexNet, referred to as \blurwavealexnet and \cblurwavealexnet, respectively, are represented in the same figure (bottom part). Note that, for a fair comparison, all three models use blur pooling in the freely-trained channels as well as deeper layers; only the Gabor channels are modified.

\section{Filter Selection and Sparse Regularization}
\label{sec:wcnn_techdetails}

We explained that, for each Gabor channel $\selectOutchannel \in \gaborchannels$, the average kernel $\avgWeightimgSelect$ is the real part of a \dtcwpt filter, as written in \eqref{eq:gaborfilter}. We now explain how the filter selection is done; in other words, how $\selectDtchannel'$ is chosen among $\setof{4 \times 4^\depth}$. Since input images are real-valued, we restrict to the filters with bandwidth located in the half-plane of positive \xcoord-values. For the sake of concision, we denote by $\dtchannels := 2 \times 4^\depth$ the number of such filters.

For any RGB image $\inpmultimg \in \twodseqs^3$, a luminance image $\colormiximg \in \twodseqs$ is computed following \eqref{eq:colormix}, using a $\obo$ convolution layer. Then, \dtcwpt is performed on $\colormiximg$. We denote by $\dtmultimg := (\dtimgSelect)_{\selectDtchannelInrange}$ the tensor containing the real part of the \dtcwpt feature maps:
\begin{equation}
    \dtimgSelect = \bigl(\colormiximg \star \Real\complexWeightimgDepthSelect\bigr) \downarrow 2^{\depth-1}.
\end{equation}
For the sake of computational efficiency, \dtcwpt is performed with a succession of subsampled separable convolutions and linear combinations of real-valued wavelet packet feature maps \cite{Selesnick2005}. To match the subsampling factor $\sub := 2^{\depth - 1}$ of the standard model, the last decomposition stage is performed without subsampling.

\subsection{Filter Selection}

The number of dual-tree feature maps $\dtchannels$ may be greater than the number of Gabor channels $\outchannelsHigh$. In that case, we therefore want to select filters that contribute the most to the network’s predictive power.
First, the low-frequency feature maps ${\dtimg}_0$ and ${\dtimg}_{(4^\depth + 1)}$ are discarded. Then, a subset of $\dtchannelsSubset < \dtchannels$ feature maps is manually selected
and permuted in order to form clusters in the Fourier domain. Considering a (truncated) permutation matrix $\permutmatrix \in \mathR^{\dtchannelsSubset \times \dtchannels}$, the output of this transformation, denoted by $\dtmultimgPerm \in \twodseqs^{\dtchannelsSubset}$, is defined by:
\begin{equation}
    \dtmultimgPerm := \permutmatrix \, \dtmultimg.
\label{eq:permut}
\end{equation}
The feature maps $\dtmultimgPerm$ are then sliced into $\ngroups$ groups of channels $\dtmultimgPermGroup \in \twodseqs^{\inchannelsGroup}$, each of them corresponding to a cluster of band-pass dual-tree filters with neighboring frequencies and orientations. On the other hand, the output of the wavelet block,
    $\outmultimgHigh := (\outimgSelect)_{\selectOutchannel \in \range{\outchannelsLow + 1}{\outchannels}} \in \twodseqs^{\outchannelsHigh}$,
where $\outimgSelect$ has been introduced in \eqref{eq:conv}, is also sliced into $\ngroups$ groups of channels $\outmultimgGroup \in \twodseqs^{\outchannelsGroup}$. Then, for each group $\selectGroupInrange$, an affine mapping between $\dtmultimgPermGroup$ and $\outmultimgGroup$ is performed. It is characterized by a trainable matrix $\featmapmixmatrixGroup := \bigl(
    \featmapmixvecGroup_1,\, \cdots,\, \featmapmixvecGroup_{\outchannelsGroup}
\bigr)^\top \in \mathR^{\outchannelsGroup \times \inchannelsGroup}$ such that, for any $\selectOutchannelInrangeGroup$,
\begin{equation}
    \outimgGroupSelect := \featmapmixvecGroupTopSelect \cdot \dtmultimgPermGroup.
\label{eq:featmapmix}
\end{equation}
As in the color mixing stage, this operation is implemented as a $\obo$ convolution layer.

A schematic representation of the real- and complex-valued wavelet blocks can be found in \cref{fig:waveblock}.

\begin{figure}
    \centering
    \includegraphics[width=0.32\columnwidth]{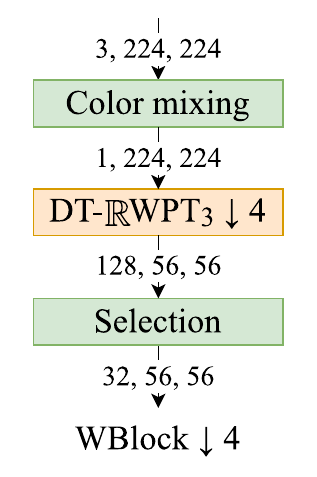}
    \hspace{15pt}
    \includegraphics[width=0.32\columnwidth]{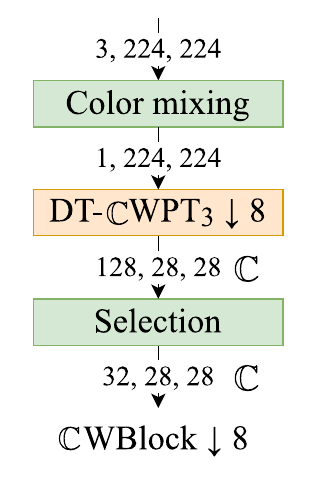}
    \caption{
        Detail of a wavelet block with $\depth = 3$ as in AlexNet, in its \rmax (left)  and \cmod (right) versions. \dtrwpt corresponds to the real part of \dtcwpt.
    }
    \label{fig:waveblock}
\end{figure}

\subsection{Sparse Regularization}

For any group $\selectGroupInrange$ and output channel $\selectOutchannelInrangeGroup$, we want the model to select one and only one wavelet packet feature map within the $\selectGroup$\nobreakdash-th group. In other words, each row vector $\featmapmixvecGroupSelect := \bigl(
    \featmapmixvalGroup_{\selectOutchannel,\, 1},\, \cdots,\, \featmapmixvalGroup_{\selectOutchannel,\, \inchannelsGroup}
\bigr)^\top$ of $\featmapmixmatrixGroup$ contains no more than one nonzero element, such that \eqref{eq:featmapmix} becomes
\begin{equation}
    \outimgGroupSelect = \featmapmixvalGroup_{\selectOutchannel\selectInchannel} \dtimgPermGroupSelect
\end{equation}
for some (unknown) value of $\selectInchannelInrangeGroup$. To enforce this property during training, we add a mixed-norm $l^1/l^\infty$-regularizer~\cite{Liu2010} to the loss function to penalize non-sparse feature map mixing as follows:
\begin{equation}
    \objfun := \objfunbis + \sumof{\selectGroup}{\ngroups} \regparamGroup \sumof{\selectOutchannel}{\outchannelsGroup} \left(\frac{\bignormone{\featmapmixvecGroupSelect}}{\bignorminfty{\featmapmixvecGroupSelect}} - 1\right),
\label{eq:regobjfun}
\end{equation}
where $\objfunbis$ denotes the standard cross-entropy loss and $\regparamvec \in \mathR^\ngroups$ denotes a vector of regularization hyperparameters. Note that the unit bias in \eqref{eq:regobjfun} serves for interpretability of the regularized loss ($\objfun = \objfunbis$ in the desired configuration) but has no impact on training.

\section{Adaptation to ResNet: Batch Normalization}
\label{sec:bnresnet}

\begin{figure*}
    \centering
    \input{inputs/models_resnet}
    \caption{
        First layers of ResNet and its variants, corresponding to a convolution layer followed by ReLU and max pooling. The bias module from \cref{fig:models} has been replaced by an affine batch normalization layer (``BN → Bias'', or ``\cmodbn → Bias'' when placed after Modulus---see \cref{sec:bnresnet}). Top: ResNet without blur pooling. Middle: Zhang's ``blurpooled'' models \cite{Zhang2019}. Bottom: Zou \etal's approach, using adaptive blur pooling \cite{Zou2023}.
    }
\label{fig:models_resnet}
\end{figure*}

In many architectures including ResNet, the bias is computed after an operation called \emph{batch normalization} (BN) \cite{Ioffe2015}. In this context, the first layers have the following structure:
\begin{equation}
    \convlayer \to \subsamplayer \to \bnlayer \to \biaslayer \to \relulayer \to \maxpoollayer.
\label{eq:cnnfirstlayers_bn}
\end{equation}
As shown hereafter, the \rmax-\cmod substitution yields, analogously to \eqref{eq:cmodmodel},
\begin{equation}
    \complexConvlayer \!\to\! \subsamplayer \!\to\! \moduluslayer \!\to\! \bnzerolayer \!\to\! \biaslayer \!\to\! \relulayer,
\label{eq:cmodmodel_bn}
\end{equation}
where \cmodbn refers to a special type of batch normalization without mean centering. A schematic representation of the \dtcwpt-based ResNet architecture and its variants is provided in \cref{fig:models_resnet}.

A BN layer is parameterized by trainable weight and bias vectors, respectively denoted by $\multvec$ and $\biasvec \in \mathR^\outchannels$. In the remaining of the section, we consider input images $\inpmultimg$ as a stack of discrete stochastic processes. Then, expression \eqref{eq:convrelumaxpool} is replaced by
\begin{equation}
    \actimgSelect \!:=\! \maxpool\!\left\{
        \relu\!\left(
            \multvalSelect \!\cdot\! \frac{
                \outimgSelect \!-\! \empiricalmeanSub[\outimgSelect]
            }{
                \sqrt{\empiricalvarSub[\outimgSelect] \!+\! \varepsilon}
            } \!+\! \biasvalSelect
        \right)\!
    \right\}\!,
\label{eq:convbnrelumaxpool}
\end{equation}
with $\outimgSelect$ satisfying \eqref{eq:conv} (output of the first convolution layer). In the above expression, we have introduced $\empiricalmeanSub(\outimgSelect) \in \mathR$ and $\empiricalvarSub(\outimgSelect) \in \mathR_+$, which respectively denote the mean expected value and variance of $\outimgSelect[\vectorindex]$, for indices $\vectorindex$ contained in the support of $\outimgSelect$, denoted by $\supp(\outimgSelect)$. Let us denote by $\imgsize \in \nonzeroMathN$ the support size of input images. Therefore, if the filter's support size $\filtersize$ is much smaller that $\imgsize$, then $\supp(\outimgSelect)$ is roughly of size $\imgsize / \sub$. We thus define the above quantities as follows:
\begin{align}
    \empiricalmeanSub[\outimgSelect] &:= \frac{\sub^2}{\imgsize^2} \sumoverVectorindices \Expval[\outimgSelect[\vectorindex]]; \\
    \empiricalvarSub[\outimgSelect] &:= \frac{\sub^2}{\imgsize^2} \sumoverVectorindices \Var[\outimgSelect[\vectorindex]].
\end{align}
In practice, estimators are computed over a minibatch of images, hence the layer's denomination. Besides, $\varepsilon > 0$ is a small constant added to the denominator for numerical stability. For the sake of concision, we now assume that $\multvec = \bone$. Extensions to other multiplicative factors is straightforward.

Let $\selectOutchannel \in \gaborchannels$ denote a Gabor channel. Then, recall that $\outimgSelect$ satisfies \eqref{eq:dtcwpt_real} (output of the \waveblock), with
\begin{equation}
    \avgWeightimgSelect := \Real \complexAvgWeightimgSelect,
\label{eq:gaborfilter0}
\end{equation}
where $\complexAvgWeightimgSelect$ denotes one of the Gabor-like filters spawned by \dtcwpt. The following proposition states that, if the kernel's bandwidth is small enough, then the output of the convolution layer sums to zero.

\vspace{5pt}

\begin{proposition}
    We assume that the Fourier transform of $\complexAvgWeightimgSelect$ is supported in a region of size $\discretefilterSupportsize \times \discretefilterSupportsize$ which does not contain the origin (Gabor-like filter). If, moreover, $\discretefilterSupportsize \leq \frac{2\pi}{\sub}$, then
    \begin{equation}
        \sumoverVectorindices \outimgSelect[\vectorindex] = 0.
    \label{eq:expval}
    \end{equation}
\label{prop:expval}
\end{proposition}

\begin{IEEEproof}
    This proposition takes advantage of Shannon's sampling theorem. A similar reasoning can be found in the proof of Theorem~2.9 in \cite{Leterme2023}.
\end{IEEEproof}

\vspace{5pt}

In practice, the power spectrum of \dtcwpt filters cannot be exactly zero on regions with nonzero measure, since they are finitely supported. However, we can reasonably assume that it is concentrated within a region of size $\pi / 2^{\depth - 1} = \pi / \sub$. Therefore, since we have discarded low-pass filters, the conditions of \cref{prop:expval} are approximately met for $\complexAvgWeightimgSelect$.

We now assume that \eqref{eq:expval} is satisfied. Moreover, we assume that $\Expval[\outimgSelect[\vectorindex]]$ is constant for any $\vectorindex \in \supp(\outimgSelect)$. Aside from boundary effects, this is true if $\Expval[\colormiximg[\vectorindex]]$ is constant for any $\vectorindex \in \supp(\colormiximg)$.
This property is a rough approximation for images of natural scenes or man-made objects. In practice, the main subject is generally located at the center, the sky at the top, \etc. These are sources of variability for color and luminance distributions across images, as discussed in \cite{Torralba2003}.

We then get, for any $\vectorindex \in \mathZ^2$, $\Expval[\outimgSelect[\vectorindex]] = 0$.
Therefore, interchanging max pooling and ReLU yields the normalized version of \eqref{eq:convmaxpoolrelu}:
\begin{equation}
    \actRmaximgSelect = \relu\left(
        \frac{
            \rmaximgSelect
        }{
            \sqrt{\empiricalmeanSub[\outimgSelect^2] + \varepsilon}
        } + \biasvalSelect
    \right).
\label{eq:convmaxpoolbnrelu}
\end{equation}

As in \cref{subsec:proposedmodels_principle}, we replace $\rmaximgSelect$ by $\cmodimgSelect$ for any Gabor channel $\selectOutchannel \in \gaborchannels$, which yields the normalized version of \eqref{eq:convmodulusrelu}:
\begin{equation}
    \actCmodimgSelect := \relu\left(
        \frac{
            \cmodimgSelect
        }{
            \sqrt{\empiricalmeanSub[\outimgSelect^2] + \varepsilon}
        } + \biasvalSelect
    \right).
\label{eq:convmodulusbnrelu}
\end{equation}

Implementing \eqref{eq:convmodulusbnrelu} as a deep learning architecture is cumbersome because $\outimgSelect$ needs to be explicitly computed and kept in memory, in addition to $\cmodimgSelect$. Instead, we want to express the second-order moment $\empiricalmeanSub[\outimgSelect^2]$ (in the denominator) as a function of $\cmodimgSelect$. To this end, we state the following proposition.

\vspace{5pt}

\begin{proposition}
    If we restrict the conditions of \cref{prop:expval} to $\discretefilterSupportsize \leq \pi / \sub$, we have
    \begin{equation}
        \normtwo{\outimgSelect}^2 = 2\bignormtwo{\cmodimgSelect}^2.
    \label{eq:var}
    \end{equation}
\label{prop:var}
\end{proposition}

\begin{IEEEproof}
    This result, once again, takes advantage of Shannon's sampling theorem. The proof of our Proposition~2.10 in \cite{Leterme2023} is based on similar arguments.
\end{IEEEproof}

\vspace{5pt}

As for \cref{prop:expval}, the conditions of \cref{prop:var} are approximately met. We therefore assume that \eqref{eq:var} is satisfied, and \eqref{eq:convmodulusbnrelu} becomes
\begin{equation}
    \actCmodimgSelect := \relu\left(
        \frac{
            \cmodimgSelect
        }{
            \sqrt{\frac12\empiricalmeanTwosub[{\cmodimgSelect}^2] + \varepsilon}
        } + \biasvalSelect
    \right).
\label{eq:convmodulusbnrelubis}
\end{equation}
In the case of ResNet, the bias layer (Bias) is therefore preceded by a batch normalization layer without mean centering satisfying \eqref{eq:convmodulusbnrelubis}, which we call \cmodbn. The second-order moment of ${\cmodimgSelect}$ is computed on feature maps which are twice smaller than $\outimgSelect$ in both directions (hence the index ``$2\sub$'' in \eqref{eq:convmodulusbnrelubis}), which is the subsampling factor for the \cmod operator.

\section{Implementation Details}
\label{sec:appendix_impl_details}

In this section, we provide further information that complements the experimental details presented in \cref{subsec:experiments_details} of the main paper.

\subsection{Subsampling Factor and Decomposition Depth}

As explained in \cref{subsec:proposedmodels_baseline}, the decomposition depth $\depth$ is chosen such that $\sub = 2^{\depth-1}$ (subsampling factor). Since $\sub = 4$ in AlexNet and $2$ in ResNet, we get $\depth = 3$ and $2$, respectively (see \cref{tab:exp_details}). Therefore, the number of dual-tree filters $\dtchannels := 2 \times 4^\depth$ is equal to $128$ and $32$, respectively.

\subsection{Number of Freely-Trained and Gabor Channels}

The split $\outchannelsLow$-$\outchannelsHigh$ between the freely-trained and Gabor channels, provided in the last row of \cref{tab:exp_details}, have been empirically determined from the standard models. More specifically, considering standard AlexNet and ResNet-34 trained on ImageNet (see \cref{subfig:convkernels_ai,subfig:convkernels_ri}, respectively), we determined the characteristics of each convolution kernel: frequency, orientation, and coherence index (which indicates whether an orientation is clearly defined). This was done by computing the \emph{tensor structure} \cite{Jahne2004}. Then, by applying proper thresholds, we isolated the Gabor-like kernels from the others, yielding the approximate values of $\outchannelsLow$ and $\outchannelsHigh$. Furthermore, this procedure allowed us to draw a rough estimate of the distribution of the Gabor-like filters in the Fourier domain, which was helpful to design the mapping scheme shown in \cref{fig:filtergrouping}, as explained below.

\begin{table}
    \caption{Experimental settings for our twin models}
    \renewcommand{\arraystretch}{1.}
    \centering\small
    \begin{tabular}{r|cc}
        \hline

        & \textbf{\wavealexnet} & \textbf{\waveresnet} \\

        \hline

        $\sub$ (subsampling factor) & $4$ & $2$ \\

        $\depth$ (decomposition depth) & $3$ & $2$ \\

        $\outchannelsLow,\, \outchannelsHigh$ (output channels) & $32,\, 32$ & $40,\, 24$ \\

        \hline
    \end{tabular}
    \vspace{-0pt}
    \label{tab:exp_details}
\end{table}

\subsection{Filter Selection and Grouping}

We then manually selected $\dtchannelsSubset < \dtchannels$ filters, used in \eqref{eq:permut}. In particular, we removed the two low-pass filters, which are outside the scope of our theoretical study. Besides, for computational reasons, in \wavealexnet we removed $32$ ``extremely'' high-frequency filters which are clearly absent from the standard model (see \cref{subfig:filtergrouping_alexnet}). Finally, in \waveresnet we removed the $14$ filters whose bandwidths outreach the boundaries of the Fourier domain $\mpipi^2$ (see \cref{subfig:filtergrouping_resnet}). These filters indeed have a poorly-defined orientation, since a small fraction of their energy is located at the far end of the Fourier domain \cite[see Fig.~1, ``Proposed \dtcwpt'']{Bayram2008}. Therefore, they somewhat exhibit a checkerboard pattern.%
\footnote{
    Note that the same procedure could have been applied to \wavealexnet, but it was deemed unnecessary because the boundary filters were spontaneously discarded during training.
}

\begin{figure}
    \centering
    \begin{subfigure}{0.6\columnwidth}
        \centering
        \includegraphics[width=\columnwidth]{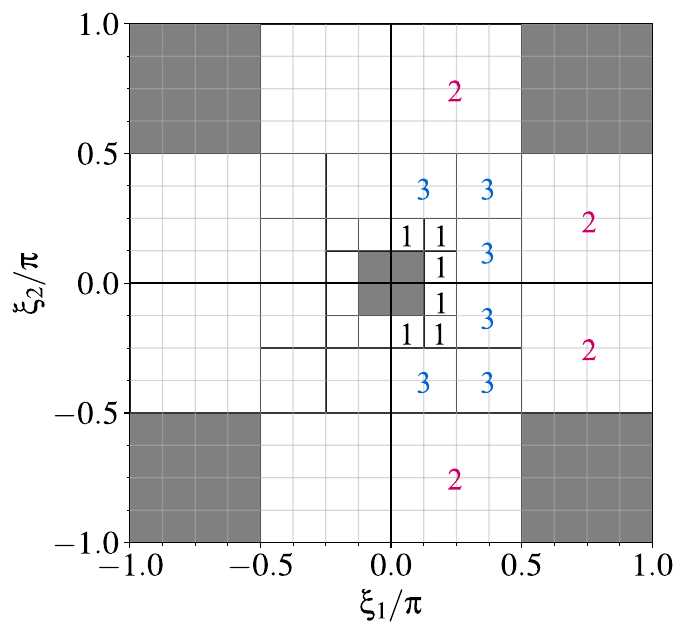}
        \caption{
            \wavealexnet ($\depth = 3$)
        }
    \label{subfig:filtergrouping_alexnet}
    \end{subfigure}
    \hspace{30pt}
    \begin{subfigure}{0.6\columnwidth}
        \centering
        \includegraphics[width=\columnwidth]{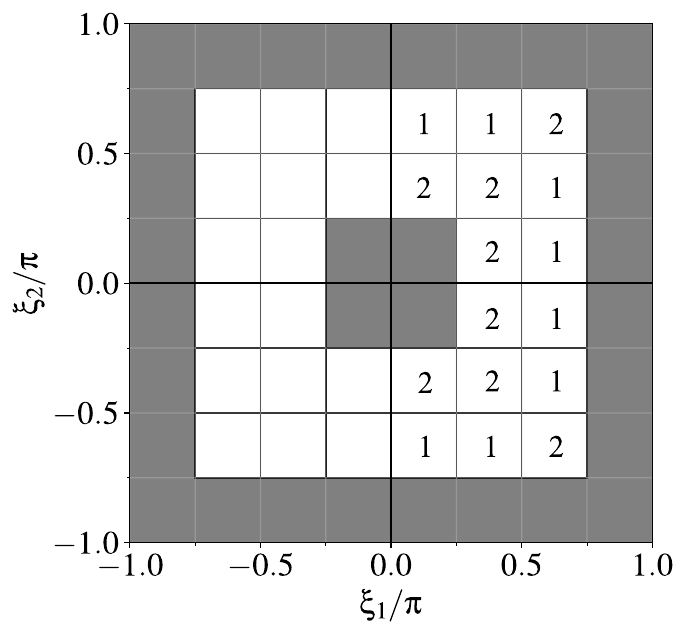}
        \caption{
            \waveresnet ($\depth = 2$)
        }
        \label{subfig:filtergrouping_resnet}
    \end{subfigure}
    \caption{
        Mapping scheme from \dtcwpt feature maps $\dtmultimg \in \twodseqs^{\dtchannels}$ to the wavelet block's output $\outmultimgHigh \in \twodseqs^{\outchannelsHigh}$. Each wavelet feature map is symbolized by a small square in the Fourier domain, where its energy is mainly located. The gray areas show the feature maps which have been manually removed. Elsewhere, each group of feature maps $\dtmultimgPermGroup \in \twodseqs^{\inchannelsGroup}$ is symbolized by a dark frame---in (b), $\inchannelsGroup$ is always equal to $1$. For each group $\selectGroup \in \setof{\ngroups}$, a number indicates how many output channels $\outchannelsGroup$ are assigned to it. The colored numbers in (a) refer to groups on which we have applied $l^\infty / l^1$-regularization. Note that, when inputs are real-valued, only the half-plane of positive \xcoord-values is considered.
    }
    \label{fig:filtergrouping}
\end{figure}

As explained in \cref{sec:wcnn_techdetails}, once the \dtcwpt feature maps have been manually selected, the output $\dtmultimgPerm \in \twodseqs^{\dtchannelsSubset}$ is sliced into $\ngroups$ groups of channels $\dtmultimgPermGroup \in \twodseqs^{\inchannelsGroup}$. For each group $\selectGroup$, a depthwise linear mapping from $\dtmultimgPermGroup$ to a bunch of output channels $\outmultimgGroup \in \twodseqs^{\outchannelsGroup}$ is performed. Finally, the wavelet block's output feature maps $\outmultimgHigh \in \twodseqs^{\outchannelsHigh}$ are obtained by concatenating the outputs $\outmultimgGroup$ depthwise, for any $\selectGroup \in \setof{\ngroups}$. \Cref{fig:filtergrouping} shows how the above grouping is made, and how many output channels $\outchannelsGroup$ each group $\selectGroup$ is assigned to.

During training, the above process aims at selecting one single \dtcwpt feature map among each group. This is achieved through mixed-norm $l^\infty / l^1$ regularization, as introduced in \eqref{eq:regobjfun}. The regularization hyperparameters $\regparamGroup$ have been chosen empirically. If they are too small, then regularization will not be effective. On the contrary, if they are too large, then the regularization term will become predominant, forcing the trainable parameter vector $\featmapmixvecGroupSelect$ to randomly collapse to $0$ except for one element. The chosen values of $\regparamGroup$ are displayed in \cref{tab:regparams}, for each group $\selectGroup$ of \dtcwpt feature maps. The groups with only one feature map do not need any regularization since this feature map is automatically selected. The second and third rows of \wavealexnet correspond to the blue and magenta groups in \cref{subfig:filtergrouping_alexnet}, respectively.

\begin{table}
    \caption{Regularization hyperparameters}
    \renewcommand{\arraystretch}{1.2}
    \centering\small
    \begin{tabular}{c|cr}
        \hline

        \multicolumn{1}{c|}{\textbf{Model}} & \multicolumn{1}{c}{\textbf{Filt.\@ frequency}} & \multicolumn{1}{c}{\textbf{Reg.\@ param.}} \\

        \hline

        \multirow{3}{*}{\wavealexnet} & $\intervalexclr{\pi / 8}{\pi / 4}$ & -- \\

        & $\intervalexclr{\pi / 4}{\pi / 2}$ & $4.1 \cdot 10^{-3}$ \\

        & $\intervalexclr{\pi / 2}{\pi}$ & $3.2 \cdot 10^{-4}$ \\

        \hline

        \waveresnet & \multicolumn{1}{c}{any} & -- \\

        \hline
    \end{tabular}
    \label{tab:regparams}
\end{table}

\subsection{Benchmark against Blur-Pooling-based Approaches}

As mentioned in \cref{subsec:proposedmodels_antialiasing}, we compare blur-pooling-based antialiasing approach with ours. To apply static or adaptive blur pooling to the \wavecnns, we proceed as follows. Following Zhang's implementation, the wavelet block is not antialiased if $\sub = 2$ as in ResNet, for computational reasons. However, when $\sub = 4$ as in AlexNet, a blur pooling layer is placed after ReLU, and the wavelet block's subsampling factor is divided by $2$. 
Moreover, max pooling is replaced by max-blur pooling.
The size of the blurring filters is set to $3$, as recommended by Zhang \cite{Zhang2019}.

\section{Accuracy vs Consistency: Additional Plots}
\label{sec:appendix_accuracy_vs_consistency}

\Cref{fig:tradeoff_resnet} shows the relationship between consistency and prediction accuracy of AlexNet and ResNet-based models on ImageNet, for different filter sizes ranging from $1$ (no blur pooling) to $7$ (heavy loss of high-frequency information). The data for AlexNet on the validation set are displayed in the main document, \cref{fig:tradeoff}.
As recommended by Zhang \cite{Zhang2019}, the optimal trade-off is generally achieved when the blurring filter size is equal to $3$. Moreover, in either case, at equivalent level of consistency, replacing blur pooling by our \cmod-based antialiasing approach in the Gabor channels increases accuracy.

\begin{figure}
    \centering
    \begin{subfigure}{0.95\columnwidth}
        \centering
        \includegraphics[width=\columnwidth]{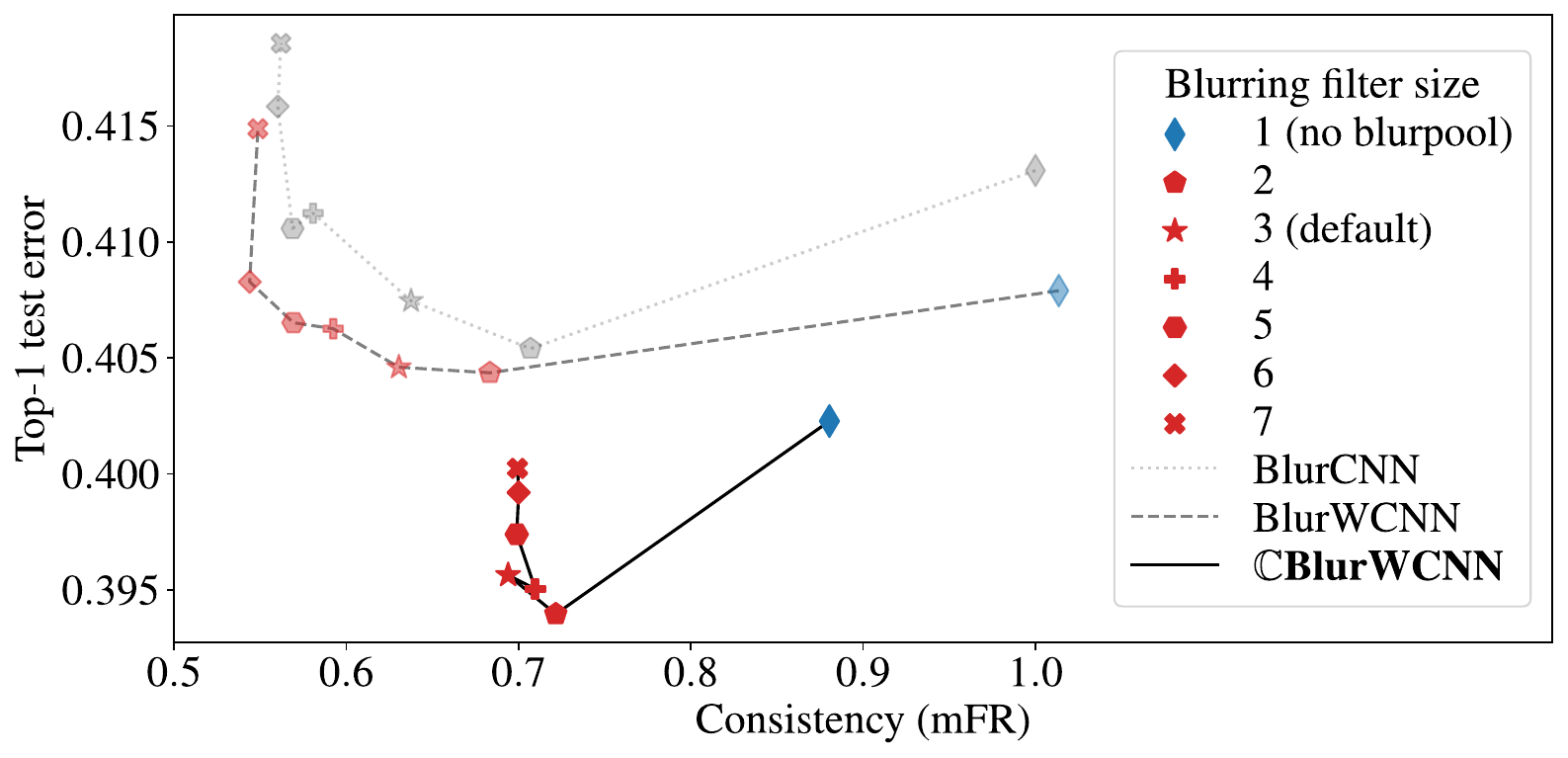}
        \caption{AlexNet, test set ($50$K images)}
        \vspace{10pt}
    \end{subfigure}
    \begin{subfigure}{0.95\columnwidth}
        \centering
        \includegraphics[width=\columnwidth]{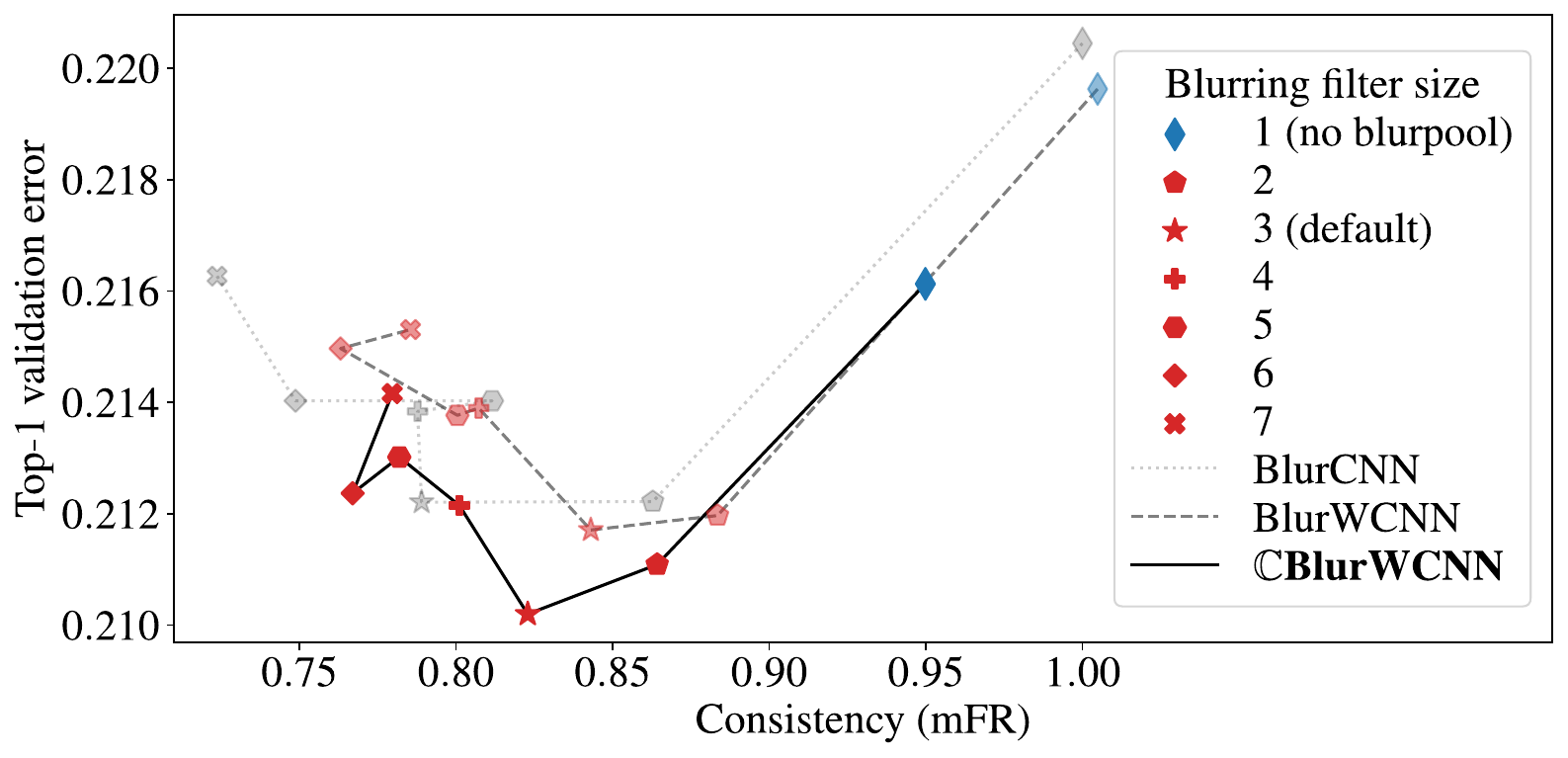}
        \caption{ResNet-34, validation set ($100$K images)}
        \vspace{10pt}
    \end{subfigure}
    \begin{subfigure}{0.95\columnwidth}
        \centering
        \includegraphics[width=\columnwidth]{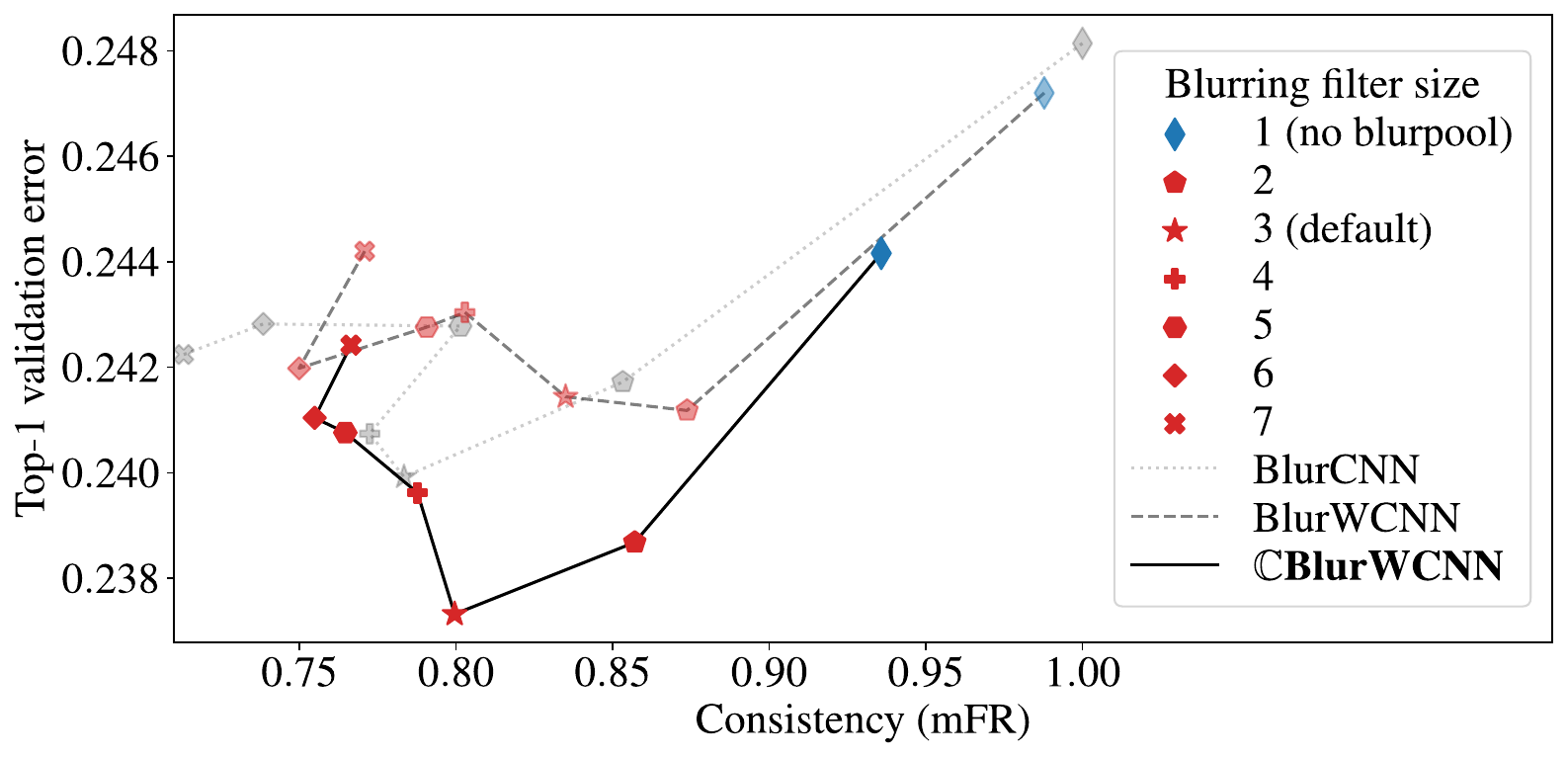}
        \caption{ResNet-34, test set ($50$K images)}
    \end{subfigure}
    \caption{
        Classification accuracy (ten-crops) vs consistency, measuring the stability of predictions to small input shifts (the lower the better for both axes). The metrics have been computed on ImageNet-1K, on both validation set ($100$K images set aside from the training set) and test set ($50$K images provided as a separate dataset). For each model (\blurcnn, \blurwavecnn and \cblurwavecnn), we increased the blurring filter size from $1$ (\ie, no blur pooling) to $7$. The blue diamonds (no blur pooling) and red stars (blur pooling with filters of size $3$) correspond to the models for which evaluation metrics have been reported in \cref{table:results_imagenet} (models trained after $90$ epochs).
    }
    \label{fig:tradeoff_resnet}
\end{figure}

\section{Computational cost}
\label{sec:appendix_computational_costs}

This section provides technical details about our estimation of the computational cost (FLOPs), such as reported in \cref{table:complexity}, for \emph{one input image} and \emph{one Gabor channel}. This metric was estimated in the case of standard 2D convolutions.



\subsection{Average Computation Time per Operation}

The following values have been determined experimentally using PyTorch (CPU computations). They have been normalized with respect to the computation time of an addition.

\begin{align*}
    \timeSum &= 1.0 \quad \mbox{(addition);} \\
    \timeProd &= 1.0 \quad \mbox{(multiplication);} \\
    \timeExp &= 0.75 \quad \mbox{(exponential);} \\
    \timeMod &= 3.5 \quad \mbox{(modulus);} \\
    \timeReLU &= 0.75 \quad \mbox{(ReLU);} \\
    \timeMax &= 12.0 \quad \mbox{(max pooling).}
\end{align*}

\subsection{Computational Cost per Layer}

In the following paragraphs, $\outchannels \in \nonzeroMathN$ denotes the number of output channels (depth) and $\imgsizebis \in \nonzeroMathN$ denotes the size of output feature maps (height and width). However, note that $\imgsizebis$ is not necessary the same for all layers. For instance, in standard ResNet, the output of the first convolution layer is of size $\imgsizebis = 112$, whereas the output of the subsequent max pooling layer is of size $\imgsizebis = 56$. For each type of layer, we calculate the number of FLOPs required to produce a single output channel $\selectOutchannelInrange$. Moreover, we assume, without loss of generality, that the model processes one input image at a time.

\paragraph{Convolution Layers}

Inputs of size $(\inchannels \times \imgsize \times \imgsize)$ (input channels, height and width); outputs of size $(\outchannels \times \imgsizebis \times \imgsizebis)$.
For each output unit, a convolution layer with kernels of size $(\filtersize \times \filtersize)$ requires $\inchannels \filtersize^2$ multiplications and $\inchannels \filtersize^2 - 1$ additions. Therefore, the computational cost per output channel is equal to
\begin{equation}
    \TimeConv = {\imgsizebis}^2 \left(
        (\inchannels \filtersize^2 - 1) \cdot \timeSum + \inchannels \filtersize^2 \cdot \timeProd
    \right).
\end{equation}

\paragraph{Complex Convolution Layers}

Inputs of size $(\inchannels \times \imgsize \times \imgsize)$; complex-valued outputs of size $(\outchannels \times \imgsizebis \times \imgsizebis)$.
For each output unit, a complex-valued convolution layer requires $2 \times \inchannels \filtersize^2$ multiplications and $2 \times (\inchannels \filtersize^2 - 1)$ additions. Computational cost per output channel:
\begin{equation}
    \TimeCConv = 2 {\imgsizebis}^2 \left(
        (\inchannels \filtersize^2 - 1) \cdot \timeSum + \inchannels \filtersize^2 \cdot \timeProd
    \right).
\end{equation}
Note that, in our implementations, the complex-valued convolution layers are less expensive than the real-valued ones, because the output size $\imgsizebis$ is twice smaller, due to the larger subsampling factor.

\paragraph{Bias and ReLU}

Inputs and outputs of size $(\outchannels \times \imgsizebis \times \imgsizebis)$. One evaluation for each output unit:

\begin{equation}
    \TimeBias = {\imgsizebis}^2 \, \timeSum
    \qqand
    \TimeReLU = {\imgsizebis}^2 \, \timeReLU.
\end{equation}

\paragraph{Max Pooling}

Outputs of size $(\outchannels \times \imgsizebis \times \imgsizebis)$, with $\imgsizebis$ depending on whether subsampling is performed at this stage (no subsampling when followed by a blur pooling layer). One evaluation for each output unit:

\begin{equation}
    \TimeMax = {\imgsizebis}^2 \, \timeMax.
\end{equation}

\paragraph{Modulus Pooling}

Complex-valued inputs and real-valued outputs of size $(\outchannels \times \imgsizebis \times \imgsizebis)$. One evaluation for each output unit:

\begin{equation}
    \TimeMod = {\imgsizebis}^2 \, \timeMod.
\end{equation}

\paragraph{Batch Normalization}

Inputs and outputs of size $(\outchannels \times \imgsizebis \times \imgsizebis)$. A batch normalization (BN) layer, described in \eqref{eq:convbnrelumaxpool}, can be split into several stages.
\begin{enumerate}
    \item Mean: ${\imgsizebis}^2$ additions.
    \item Standard deviation: ${\imgsizebis}^2$ multiplications, ${\imgsizebis}^2$ additions (second moment), ${\imgsizebis}^2$ additions (subtract squared mean).
    \item Final value: ${\imgsizebis}^2$ additions (subtract mean), $2{\imgsizebis}^2$ multiplications (divide by standard deviation and multiplicative coefficient).
\end{enumerate}
Overall, the computational cost per image and output channel of a BN layer is equal to
\begin{equation}
    \TimeBN = {\imgsizebis}^2 \left(
        4 \, \timeSum + 3 \, \timeProd
    \right).
\end{equation}

\paragraph{Static Blur Pooling}

Inputs of size $(\outchannels \times 2\imgsizebis \times 2\imgsizebis)$; outputs of size $(\outchannels \times \imgsizebis \times \imgsizebis)$.
For each output unit, a static blur pooling layer \cite{Zhang2019} with filters of size $(\blurfiltsize \times \blurfiltsize)$ requires $\blurfiltsize^2$ multiplications and $\blurfiltsize^2 - 1$ additions. The computational cost per output channel is therfore equal to
\begin{equation}
    \TimeBlur = {\imgsizebis}^2 \left(
        (\blurfiltsize^2 - 1) \cdot \timeSum + \blurfiltsize^2 \cdot \timeProd
    \right).
\end{equation}

\paragraph{Adaptive Blur Pooling}

Inputs of size $(\outchannels \times 2\imgsizebis \times 2\imgsizebis)$; outputs of size $(\outchannels \times \imgsizebis \times \imgsizebis)$.
An adaptive blur pooling layer \cite{Zou2023} with filters of size $(\blurfiltsize \times \blurfiltsize)$ splits the $\outchannels$ output channels into $\ngroups := \outchannels / \outchannelsPergroup$ groups of $\outchannelsPergroup$ channels that share the same blurring filters. The adaptive blur pooling layer can be decomposed into the following stages.
\begin{enumerate}
    \item Generation of blurring filters using a convolution layer with trainable kernels of size $(\blurfiltsize \times \blurfiltsize)$: inputs of size $(\outchannels \times 2\imgsizebis \times 2\imgsizebis)$, outputs of size $(\ngroups\blurfiltsize^2 \times \imgsizebis \times \imgsizebis)$. For each output unit, this stage requires $\outchannels \blurfiltsize^2$ multiplications and $\outchannels \blurfiltsize^2 - 1$ additions. The computational cost divided by the number $\outchannels$ of channels is therefore equal to
    \begin{equation}
        \Computtime_{\idxconv\idxadablur} = {\imgsizebis}^2 \, \frac{\blurfiltsize^2}{\outchannelsPergroup} \left(
            (\outchannels \blurfiltsize^2 - 1) \cdot \timeSum + \outchannels \blurfiltsize^2 \cdot \timeProd
        \right).
    \end{equation}
    Note that, despite being expressed on a per-channel basis, the above computational cost depends on the number $\outchannels$ of output channels. This is due to the asymptotic complexity of this stage in $\bigO(\outchannels^2)$.

    \item Batch normalization, inputs and outputs of size $(\ngroups\blurfiltsize^2 \times \imgsizebis \times \imgsizebis)$:
    \begin{equation}
        \Computtime_{\idxbn\idxadablur} = {\imgsizebis}^2 \, \frac{\blurfiltsize^2}{\outchannelsPergroup} \left(
            4 \, \timeSum + 3 \, \timeProd
        \right).
    \end{equation}

    \item Softmax along the depthwise dimension:
    \begin{equation}
        \Computtime_{\idxsoftmax\idxadablur} = {\imgsizebis}^2 \, \frac{\blurfiltsize^2}{\outchannelsPergroup} (\timeExp + \timeSum + \timeProd).
    \end{equation}

    \item Blur pooling of input feature maps, using the filter generated at stages (1)--(3): inputs of size $(\outchannels \times 2\imgsizebis \times 2\imgsizebis)$, outputs of size $(\outchannels \times \imgsizebis \times \imgsizebis)$. The computational cost per output channel is identical to the static blur pooling layer, even though the weights may vary across channels and spatial locations:
    \begin{equation}
        \TimeBlur = {\imgsizebis}^2 \left(
            (\blurfiltsize^2 - 1) \cdot \timeSum + \blurfiltsize^2 \cdot \timeProd
        \right).
    \end{equation}
\end{enumerate}
Overall, the computational cost of an adaptive blur pooling layer per input image and output channel is equal to
\begin{multline}
    \TimeAdablur = {\imgsizebis}^2 \, \frac{\blurfiltsize^2}{\outchannelsPergroup} \left[
        \left(
            (\outchannels + 1) \blurfiltsize^2 + 3
        \right) \cdot \timeSum
    \right. \\
    \left.
        + \left(
            (\outchannels + 1) \blurfiltsize^2 + 4
        \right) \cdot \timeProd
        + \timeExp
    \right].
\end{multline}
We notice that an adaptive blur pooling layer has an asymptotic complexity in $\bigO(\blurfiltsize^4)$, versus $\bigO(\blurfiltsize^2)$ for static blur pooling.

\subsection{Application to AlexNet- and ResNet-based Models}

Since they are normalized by the computational cost of standard models, the FLOPs reported in \cref{table:complexity} only depend on the size of the convolution kernels and blur pooling filters, respectively denoted by $\filtersize$ and $\blurfiltsize \in \nonzeroMathN$. In addition, the computational cost of the adaptive blur pooling layer depend on the number of output channels $\outchannels$ as well as the number of output channels per group $\outchannelsPergroup$.

In practice, $\filtersize$ is respectively equal to $11$ and $7$ for AlexNet- and ResNet-based models. Moreover, $\blurfiltsize = 3$, $\outchannels = 64$ and $\outchannelsPergroup = 8$. Actually, the computational cost is largely determined by the convolution layers, including step (1) of adaptive blur pooling.

\section{Memory Footprint}

This section provides technical details about our estimation of the memory footprint for \emph{one input image} and \emph{one output channel}, such as reported in \cref{table:complexity}. This metric is generally difficult to estimate, and is very implementation-dependent. Hereafter, we consider the size of the output tensors, as well as intermediate tensors saved by \texttt{torch.autograd} for the backward pass. However, we didn't take into account the tensors containing the trainable parameters.
To get the size of intermediate tensors, we used the Python package PyTorchViz.\footnote{%
    \url{https://github.com/szagoruyko/pytorchviz}
}
These tensors are saved according to the following rules.
\begin{itemize}
    \item Convolution (Conv), batch normalization (BN), Bias, max pooling (MaxPool or Max), blur pooling (BlurPool), and Modulus: the input tensors are saved, not the output. When Bias follows Conv or BN, no intermediate tensor is saved.
    \item ReLU, Softmax: the output tensors are saved, not the input.
    \item If an intermediate tensor is saved at both the output of a layer and the input of the next layer, its memory is not duplicated. An exception is Modulus, which stores the input feature maps as complex numbers.
    \item MaxPool or Max: a tensor of indices is kept in memory, indicating the position of the maximum values. The tensors are stored as 64-bit integers, so they weight twice as much as conventional float-32 tensors.
    \item BN: four 1D tensors of length $\outchannels$ are kept in memory: computed mean and variance, and running mean and variance. For BN0 \eqref{eq:convmodulusbnrelubis}, where the variance is not computed, only two tensors are kept in memory.
\end{itemize}

In the following paragraphs, we denote by $\outchannels$ the number of output channels, $\imgsize$ the size of input images (height and width), $\sub$ the subsampling factor of the baseline models ($4$ for AlexNet, $2$ for ResNet), $\blurfiltsize$ the blurring filter size (set to $3$ in practice). For each model, a table contains the size of all saved intermediate or output tensors. For example, the values associated to ``Layer1 $\to$ Layer2'' correspond to the depth (number of channel), height and width of the intermediate tensor between Layer1 and Layer2.

\subsection{AlexNet-based Models}

\paragraph{No Antialiasing}

$$\convlayer \to \biaslayer \to \relulayer \to \maxpoollayer.$$
\begin{center}
    \footnotesize
    \begin{tabular}{r|ccc}
        \hline
        ReLU $\to$ MaxPool & $\outchannels$ & $\frac{\imgsize}{\sub}$ & $\frac{\imgsize}{\sub}$ \\
        MaxPool $\to$ \textit{output} & $\outchannels$ & $\frac{\imgsize}{2\sub}$ & $\frac{\imgsize}{2\sub}$ \\
        MaxPool indices ($\times 2$) & $\outchannels$ & $\frac{\imgsize}{2\sub}$ & $\frac{\imgsize}{2\sub}$ \\
        \hline
    \end{tabular} 
\end{center}
The memory footprint for each output channel is equal to
\begin{equation*}
    \implies \TensorsizeStd = \frac74 \frac{\imgsize^2}{\sub^2}.
\end{equation*}

\paragraph{Static Blur Pooling}

$$\convlayer \to \biaslayer \to \relulayer \to \blurpoollayer \to \maxlayer \to \blurpoollayer.$$
\begin{center}
    \footnotesize
    \begin{tabular}{r|ccc}
        \hline
        ReLU $\to$ BlurPool & $\outchannels$ & $\frac{2\imgsize}{\sub}$ & $\frac{2\imgsize}{\sub}$ \\
        BlurPool $\to$ Max & $\outchannels$ & $\frac{\imgsize}{\sub}$ & $\frac{\imgsize}{\sub}$ \\
        Max $\to$ BlurPool & $\outchannels$ & $\frac{\imgsize}{\sub}$ & $\frac{\imgsize}{\sub}$ \\
        Max indices ($\times 2$) & $\outchannels$ & $\frac{\imgsize}{\sub}$ & $\frac{\imgsize}{\sub}$ \\
        BlurPool $\to$ \textit{output} & $\outchannels$ & $\frac{\imgsize}{2\sub}$ & $\frac{\imgsize}{2\sub}$ \\
        \hline
    \end{tabular} 
\end{center}
\begin{equation*}
    \implies \TensorsizeBlur = \frac{33}{4} \frac{\imgsize^2}{\sub^2}.
\end{equation*}

\paragraph{\cmod-based Approach}

$$\mathC\convlayer \to \moduluslayer \to \biaslayer \to \relulayer.$$
\begin{center}
    \footnotesize
    \begin{tabular}{r|ccc}
        \hline
        \complexConv $\to$ Modulus & $2\outchannels$ & $\frac{\imgsize}{2\sub}$ & $\frac{\imgsize}{2\sub}$ \\
        Modulus $\to$ Bias & $\outchannels$ & $\frac{\imgsize}{2\sub}$ & $\frac{\imgsize}{2\sub}$ \\
        ReLU $\to$ \textit{output} & $\outchannels$ & $\frac{\imgsize}{2\sub}$ & $\frac{\imgsize}{2\sub}$ \\
        \hline
    \end{tabular} 
\end{center}
\begin{equation*}
    \implies \TensorsizeCMod = \frac{\imgsize^2}{\sub^2}.
\end{equation*}

\subsection{ResNet-based Models}

\setcounter{paragraph}{0} 
\paragraph{No Antialiasing}

$$\convlayer \to \bnlayer \to \biaslayer \to \relulayer \to \maxpoollayer.$$
\begin{center}
    \footnotesize
    \begin{tabular}{r|ccc}
        \hline
        Conv $\to$ BN & $\outchannels$ & $\frac{\imgsize}{\sub}$ & $\frac{\imgsize}{\sub}$ \\
        BN metrics & $4\outchannels$ & -- & -- \\
        ReLU $\to$ MaxPool & $\outchannels$ & $\frac{\imgsize}{\sub}$ & $\frac{\imgsize}{\sub}$ \\
        MaxPool $\to$ \textit{output} & $\outchannels$ & $\frac{\imgsize}{2\sub}$ & $\frac{\imgsize}{2\sub}$ \\
        MaxPool indices ($\times 2$) & $\outchannels$ & $\frac{\imgsize}{2\sub}$ & $\frac{\imgsize}{2\sub}$ \\
        \hline
    \end{tabular} 
\end{center}
\begin{equation*}
    \implies \TensorsizeStd = \frac{11}{4} \frac{\imgsize^2}{\sub^2} + 4 \approx \frac{11}{4} \frac{\imgsize^2}{\sub^2}.
\end{equation*}

\paragraph{Static Blur Pooling}

$$\convlayer \to \bnlayer \to \biaslayer \to \relulayer \to \maxlayer \to \blurpoollayer.$$
\begin{center}
    \footnotesize
    \begin{tabular}{r|ccc}
        \hline
        Conv $\to$ BN & $\outchannels$ & $\frac{\imgsize}{\sub}$ & $\frac{\imgsize}{\sub}$ \\
        BN metrics & $4\outchannels$ & -- & -- \\
        ReLU $\to$ Max & $\outchannels$ & $\frac{\imgsize}{\sub}$ & $\frac{\imgsize}{\sub}$ \\
        Max $\to$ BlurPool & $\outchannels$ & $\frac{\imgsize}{\sub}$ & $\frac{\imgsize}{\sub}$ \\
        Max indices ($\times 2$) & $\outchannels$ & $\frac{\imgsize}{\sub}$ & $\frac{\imgsize}{\sub}$ \\
        BlurPool $\to$ \textit{output} & $\outchannels$ & $\frac{\imgsize}{2\sub}$ & $\frac{\imgsize}{2\sub}$ \\
        \hline
    \end{tabular}
\end{center}
\begin{equation*}
    \implies \TensorsizeBlur = \frac{21}{4} \frac{\imgsize^2}{\sub^2} + 4 \approx \frac{21}{4} \frac{\imgsize^2}{\sub^2}.
\end{equation*}

\paragraph{Adaptive Blur Pooling}

$$\convlayer \to \bnlayer \to \biaslayer \to \relulayer \to \maxlayer \to \adablurpoollayer.$$
\begin{center}
    \footnotesize
    \begin{tabular}{r|ccc}
        \hline
        Conv $\to$ BN & $\outchannels$ & $\frac{\imgsize}{\sub}$ & $\frac{\imgsize}{\sub}$ \\
        BN metrics & $4\outchannels$ & -- & -- \\
        ReLU $\to$ Max & $\outchannels$ & $\frac{\imgsize}{\sub}$ & $\frac{\imgsize}{\sub}$ \\
        Max $\to$ ABlurPool & $\outchannels$ & $\frac{\imgsize}{\sub}$ & $\frac{\imgsize}{\sub}$ \\
        Max indices ($\times 2$) & $\outchannels$ & $\frac{\imgsize}{\sub}$ & $\frac{\imgsize}{\sub}$ \\
        ABlurPool $\to$ \textit{output} & $\outchannels$ & $\frac{\imgsize}{2\sub}$ & $\frac{\imgsize}{2\sub}$ \\
        \hline
        \multicolumn{4}{c}{\textbf{Generate adaptive blurring filter}} \\
        \multicolumn{4}{c}{$\convlayer \to \bnlayer \to \biaslayer \to \softmaxlayer$} \\
        \hline
        Conv $\to$ BN & $\frac{\outchannels\blurfiltsize^2}{\outchannelsPergroup}$ & $\frac{\imgsize}{2\sub}$ & $\frac{\imgsize}{2\sub}$ \\
        BN metrics & $4\frac{\outchannels\blurfiltsize^2}{\outchannelsPergroup}$ & -- & -- \\
        Softmax $\to$ \textit{output} & $\frac{\outchannels\blurfiltsize^2}{\outchannelsPergroup}$ & $\frac{\imgsize}{2\sub}$ & $\frac{\imgsize}{2\sub}$ \\
        \hline
    \end{tabular} 
\end{center}
\begin{align*}
    \implies \TensorsizeAdablur
        &= \frac{21}{4} \frac{\imgsize^2}{\sub^2} + 4 + \frac{\blurfiltsize^2}{\outchannelsPergroup} \left(
        \frac{\imgsize^2}{2\sub^2} + 4
    \right) \\
    &\approx \frac{21}{4} \frac{\imgsize^2}{\sub^2} + \frac{\blurfiltsize^2}{\outchannelsPergroup}\frac{\imgsize^2}{2\sub^2}.
\end{align*}

\paragraph{\cmod-based Approach}

$$\mathC\convlayer \to \moduluslayer \to \bnlayer 0 \to \biaslayer \to \relulayer.$$
\begin{center}
    \footnotesize
    \begin{tabular}{r|ccc}
        \hline
        $\mathC\convlayer \to \moduluslayer$ & $2\outchannels$ & $\frac{\imgsize}{2\sub}$ & $\frac{\imgsize}{2\sub}$ \\
        $\moduluslayer \to \bnzeroerolayer$ & $\outchannels$ & $\frac{\imgsize}{2\sub}$ & $\frac{\imgsize}{2\sub}$ \\
        BN0 metrics & $2\outchannels$ & -- & -- \\
        $\relulayer \to$ \textit{output} & $\outchannels$ & $\frac{\imgsize}{2\sub}$ & $\frac{\imgsize}{2\sub}$ \\
        \hline
    \end{tabular} 
\end{center}
\begin{equation*}
    \implies \TensorsizeCMod = \frac{\imgsize^2}{\sub^2} + 2 \approx \frac{\imgsize^2}{\sub^2}.
\end{equation*}

\bibliographystyle{IEEEtran}
\bibliography{IEEEabrv,refs}

\end{document}